\pdfoutput=1
\documentclass{article}

\usepackage[utf8]{inputenc}
\usepackage[T1]{fontenc}
\usepackage[margin=1in]{geometry}
\usepackage{parskip}
\usepackage{url}
\usepackage{wrapfig}
\usepackage{booktabs}
\usepackage{amsfonts}
\usepackage{nicefrac}
\usepackage{microtype}
\usepackage{multicol}
\usepackage{graphicx}
\usepackage{amsthm}
\usepackage{bbm}
\usepackage{enumitem}
\usepackage{bm}
\usepackage[export]{adjustbox}
\usepackage{stmaryrd}
\usepackage{xifthen}
\usepackage{xargs}
\usepackage{amsmath,amssymb}
\usepackage{mathtools}
\usepackage{dsfont}
\usepackage{pifont}
\usepackage[dvipsnames]{xcolor}
\usepackage{algorithm,algorithmic}
\usepackage[authoryear]{natbib}
\usepackage{caption}
\usepackage{subcaption}
\usepackage{float}
\usepackage{multirow}
\usepackage[colorlinks=true,citecolor=blue,urlcolor=blue,linkcolor=blue]{hyperref}
\usepackage{cleveref}
\newtheorem{theorem}{Theorem}
\newtheorem{lemma}{Lemma}
\newtheorem{corollary}{Corollary}
\newtheorem{proposition}{Proposition}
\theoremstyle{definition}

\newtheorem{example}{Example}

\newcommand{\esp}{\mathbb{E}}
\newcommand{\one}{\mathbbm{1}}

\newcommand{\bcal}{\mathcal{B}}
\newcommand{\calW}{\mathcal{W}}
\newcommand{\fcal}{\mathcal{F}}
\newcommand{\kcal}{\mathcal{K}}
\newcommand{\ncal}{\mathcal{N}}
\newcommand{\rset}{\mathbb{R}}
\newcommand{\run}{\mathrm{Run}}
\newcommand{\testrun}{\run_{\mathrm{test}}}
\newcommand{\es}{\mathrm{e}}
\newcommand{\cs}{\mathrm{c}}
\newcommand{\vcal}{\mathcal{V}}
\newcommand{\xcal}{\mathcal{X}}

\DeclareMathOperator{\KL}{KL}

\title{\vspace{-1cm}Valid Stopping in Adaptive Generator--Verifier Loops}
\author{
  Mahmoud Hegazy$^{1,2}$ \and Michael I. Jordan$^{2,3}$ \and Aymeric Dieuleveut$^{1}$ \\
  \\
  $^1$ CMAP, \'Ecole polytechnique, IP Paris, France \\
  $^2$ Inria Paris, France \qquad $^3$ EECS, UC Berkeley, USA
}
\date{}

\begin{document}
\maketitle

\begin{abstract}
  Numerous agentic workflows are based on a generator--verifier loop: a generator proposes candidates, a cheap verifier scores them, and the workflow terminates when a proposal is verified as good enough. The verifier typically proxies a more costly ground-truth oracle, and as the generator searches adaptively against it, false acceptances may accumulate. Proposals can pass the proxy but fail under the costlier ground-truth check. We study when to stop these loops while controlling the false discovery rate of the accepted proposals. Our construction introduces tools of independent interest in distribution-free statistical testing and conformal risk control, including analysis of $e$-values constructed through index betting and a novel conformal risk control procedure for non-monotone losses. We validate the approach in synthetic settings and on a protein-design benchmark.
\end{abstract}

\section{Introduction}\label{sec:introduction}

A large class of inference-time systems operates as \emph{generator--verifier loops}. A generator proposes candidates, a verifier supplies scores or pass/fail feedback, and the search stops when one or more candidates look acceptable. This template appears in code generation with unit-test filtering \citep{chen2021evaluating,li2022alphacode}, theorem proving with proof checkers \citep{polu2020generating}, and mathematical reasoning with learned verifiers \citep{cobbe2021training}. Related iterative self-feedback and verbal-reflection methods revise candidates using model-generated feedback \citep{madaan2023selfrefine,shinn2023reflexion}. The advantage of using verifiers can be substantial from a computational and economic point of view because, unlike the ground-truth oracle, the verifier is cheap and fast to query repeatedly.

The substitution of a verifier for a ground-truth oracle, however, also induces a complex set of statistical and strategic problems. Once the generator adapts to past verifier outputs, it can search for candidates that exploit the proxy's weaknesses. A verifier that is reliable marginally, or on isolated queries, can accumulate false acceptances under reuse. Recent analyses of inference-time search make this over-optimization phenomenon explicit \citep{khalaf2025inference,huang2025is}. A key problem is thus that of stopping a generator--verifier loop while controlling the quality of the accepted set as measured with the oracle, not merely with the proxy.

We present a statistical procedure that halts after collecting $\kappa \ge 1$ accepted proposals while controlling the false discovery rate (FDR) of the accepted set---i.e., the expected fraction of accepted proposals that the oracle would reject---below a user-specified level $\alpha$. Using a small set of historical runs, we show how to deploy this procedure on an unobserved run where the true correctness labels are hidden. Formally, at round $t$, the generator proposes $X_t$, the verifier returns a score $S_t$ and feedback $V_t$, with the oracle label $Y_t \in \{0,1\}$  unobserved. We place no restriction on how $X_t$ depends on the past history $H_{t-1}=(X_1,V_1,S_1,\ldots,X_{t-1},V_{t-1},S_{t-1})$. The procedure returns a set of accepted proposals among $X_1, \dots, X_t$ with cardinality smaller than $\kappa$. This covers regimes in which the generator adapts aggressively to prior failures and verifier corrections. When $\kappa=1$, the guarantee controls the FDR of the first nonempty returned set, which in the commonly employed single-return setting is exactly the probability of a false first acceptance.

\paragraph{Contributions.} We solve this stopping problem by combining three ingredients. First, as shown in Section~\ref{sec:marginal}, we construct marginally valid local $e$-values through an index-betting normalization \citep{koning2023measuring}. The same section gives a finite-sample power analysis with matching upper and lower bounds under a fractional-moment condition on the likelihood ratio, and a weaker Kullback--Leibler (KL) divergence condition under which the finite-batch normalization penalty vanishes. Second, Section~\ref{sec:crc} introduces \emph{robust-core} conformal risk control (CRC), a finite-sample calibration method for CRC with non-monotone losses. Third, in Section~\ref{sec:fdr} we show how to combine these local $e$-values with the weighted online $e$-Benjamini--Hochberg ($e$-BH) procedure \citep{fischer2024onlineebh} for stopped FDR control under arbitrary within-run dependence and how this addresses the general generator--verifier stopping problem.

\paragraph{Outline.} Section~\ref{sec:overview} formalizes the run-level model and gives a quick overview of the full approach. Section~\ref{sec:related} positions our construction relative to related literature. Section~\ref{sec:marginal} and Section~\ref{sec:crc} formally introduce our $e$-value construction and robust-core CRC. Section~\ref{sec:fdr} presents the full procedure together with a synthetic illustration. Section~\ref{sec:experiments} reports experiments on a MotifBench-derived protein-design benchmark \citep{zheng2025motifbench}. In that benchmark, RFdiffusion \citep{watson2023novo} or Proteina \citep{geffner2025proteina} generates candidate backbones, ProteinMPNN \citep{dauparas2022robust} designs sequences, ESMFold \citep{lin2023evolutionary} provides a cheap structural proxy, and AlphaFold~3 \citep{abramson2024accurate} serves as the more expensive structural-prediction oracle that defines the benchmark label.

\section{Setting and Overview}\label{sec:overview}

Each looped episode is referred to as a \emph{run}, which is indexed by rounds $t=1,\ldots,T$. We note that our approach can be extended to runs of variable length. Nonetheless, we restrict our attention to a fixed horizon $T$ for simplicity. At round $t$, the generator outputs a proposal $X_t \in \xcal$, the verifier returns feedback $V_t \in \vcal$ together with a score $S_t\in\mathcal{S}\subseteq\mathbb{R}$, and the ground-truth oracle would assign a label $Y_t\in\{0,1\}$, where $Y_t=1$ means that the proposal is correct. This label is generally not available during the run. Define the \emph{history}
\begin{equation*}
  H_t \coloneqq (X_1,V_1,S_1,\ldots,X_t,V_t,S_t),
\end{equation*}
and set $\fcal_t \coloneqq \sigma(H_t)$ to be the induced filtration. Conditionally on $\fcal_{t-1}$, the tuple $(X_t,V_t,S_t,Y_t)$ is drawn from a kernel $\kcal_t(\cdot\mid H_{t-1})$. We impose no structural assumptions on dependence within a run. In particular, the generator may adapt arbitrarily to past verifier outputs. All probability and expectation operators are taken with respect to the resulting joint law of the trajectory.

\paragraph{Offline runs and score windows.} We assume access to a log of labeled historical runs,
\begin{equation*}
  \bcal=\{\run_1,\ldots,\run_n\},
  \qquad
  \run_i \coloneqq
  \bigl((X_{t,i},V_{t,i},S_{t,i},Y_{t,i})\bigr)_{t=1}^{T},
\end{equation*}
drawn i.i.d.\ at the \emph{run} level. The independence assumption applies across runs only. Within each run, the kernels $\kcal_t$ may depend arbitrarily on the past. Deployment consists of an unobserved run $\testrun$ from the same run-level distribution, with hidden labels $Y_{1,\mathrm{test}},\ldots,Y_{T,\mathrm{test}}$. At time $t\in \{1, \ldots, T\}$, the statistical task is to decide using the labelled data whether to accept the current proposal $X_{t,\mathrm{test}}$.

For a fixed context length $K\ge 0$, we restrict our attention to procedures whose round-$t$ decision depends on the most recent $K{+}1$ verifier scores. For each run $i\in \{1, \dots, n\}\cup\{\mathrm{test}\}$, define
\begin{equation*}
  W_{t,i} \coloneqq (S_{t-K,i},\ldots,S_{t,i}) \in \mathcal S^{K+1}.
\end{equation*}
We use the convention $S_{s,i}=0$ for $s\le 0$. When the run index is omitted, quantities refer to the unobserved run $\testrun$. For example, $W_t\equiv W_{t,\mathrm{test}}$ and $X_t\equiv X_{t,\mathrm{test}}$. For each round $t$, we form the augmented labeled sample
\begin{equation*}
  \calW_t \coloneqq
  \bigl\{(W_{t,1},Y_{t,1}),\ldots,(W_{t,n},Y_{t,n}),(W_t,0)\bigr\},
\end{equation*}
where the pseudo-label $0$ is attached to $W_t$ in lieu of its unobserved correctness label. We replace the hidden label with $0$ as we use the convention that the null hypothesis is that the proposal is incorrect. \paragraph{Local evidence via $e$-values.} At round $t$, we ask whether the current score window $W_t$ is more likely to correspond to a correct proposal or an incorrect one. Let $P_{0,t}$ and $P_{1,t}$ denote the marginal laws of $W_t$ conditional on $Y_t=0$ and $Y_t=1$, respectively, and write $\esp_{j,t}$ for expectation under $P_{j,t}$, $j\in\{0,1\}$. The local problem is the simple-versus-simple hypothesis test
\begin{equation*}
  H_{0,t}: W_t \sim P_{0,t}
  \qquad\text{against}\qquad
  H_{1,t}: W_t \sim P_{1,t}.
\end{equation*}
Using the labeled sample $\calW_t$, we construct an $e$-value for $H_{0,t}$. That is, we build a nonnegative statistic $E_t(\calW_t)$ such that $\esp_{0,t}\bigl[E_t(\calW_t)\bigr] \le 1$. Larger values then provide stronger evidence against $H_{0,t}$. If $P_{1,t}$ is absolutely continuous with respect to $P_{0,t}$, denoted by $P_{1,t} \ll P_{0,t}$, the ideal test statistic is the likelihood ratio $L_t = \mathrm{d}P_{1,t}/\mathrm{d}P_{0,t}$. It satisfies $\esp_{0,t}[L_t(W_t)] = 1$, maximizes $\esp_{1,t}[\log L_t(W_t)]$, and is therefore log-optimal among $W_t$-measurable valid $e$-values \citep{shafer2021testing,larsson2024numeraire}. In our setting, however, $L_t$ is unknown and must be learned from $\calW_t$. A direct plug-in estimate $\widehat L_{t,-\mathrm{test}}(W_t)$ trained on $\calW_t\setminus\{(W_{t,\mathrm{test}},0)\}$ is not necessarily valid, since $\esp_{0,t}[\widehat L_{t,-\mathrm{test}}(W_t)]$ can exceed $1$.

Our construction restores validity by wrapping nonnegative, permutation-equivariant scoring rules with index betting in the sense of \citet{koning2023measuring}. The key step is to normalize the unobserved-run score against a batch of offline null windows in a way that preserves exchangeability. We only require the fitting rule to be invariant to the order of its training sample. The resulting vector of bets on the batch of windows is then permutation-equivariant, yielding a statistic that is valid under $H_{0,t}$ for any such fitting rule. We defer the formal construction and power analysis to Section~\ref{sec:marginal}.

\paragraph{Sequential testing.} Once these roundwise $e$-values are available, the problem is reducible to an online multiple-testing task. At round $t$, we return an accepted set $R_t\subseteq\{1,\ldots,t\}$ and measure its false discovery proportion (FDP) by
\begin{equation*}
  \mathrm{FDP}(R_t)
  \coloneqq
  \frac{\sum_{s\in R_t}\one\{Y_s{=}0\}}{|R_t|\vee 1}.
\end{equation*}
For a target acceptance count $\kappa\ge 1$, define $\tau_\kappa\coloneqq\inf\{t:|R_t|\ge\kappa\}\wedge T$. Our goal is trajectory-level FDR control along the unobserved run, uniformly over stopping times up to $\tau_\kappa$:
\begin{equation}\label{eq:trajectory-fdr}
  \esp[\mathrm{FDP}(R_\tau)] \le \alpha
  \qquad \text{for every stopping time $\tau\le\tau_\kappa$}.
\end{equation}
We obtain this guarantee by aggregating the local $e$-values with the weighted online $e$-BH procedure of \citet{fischer2024onlineebh}, an online analogue of $e$-BH \citep{wang2020false}. Under arbitrary dependence across time, online $e$-BH guarantees
\begin{equation*}
  \esp\Big[\sup_{t\ge 1}\mathrm{FDP}(R_t)\Big]\le \alpha,
\end{equation*}
using only the marginal validity condition $\esp_{0,t}[E_t]\le 1$ at each round. Since $\mathrm{FDP}(R_\tau)\le \sup_{t\ge 1}\mathrm{FDP}(R_t)$ for every stopping time $\tau$, this directly implies the bound in Equation~(\ref{eq:trajectory-fdr}).

Online $e$-BH operates through an $\alpha$-spending mechanism, allocating the per-run error budget across rounds according to a nonnegative schedule $(\gamma_t)_{t\ge1}$ with $\sum_{t\ge1}\gamma_t=1$. At time $t$, we accept the proposals indexed by
\begin{equation*}
  R_t
  =
  \left\{s\le t:
    E_s\ge\frac{1}{\alpha\gamma_s k_t}
  \right\},
\end{equation*}
where $k_t$ is the largest $k\in\{1,\ldots,t\}$ for which at least $k$ indices $s\le t$ satisfy $E_s\ge1/(\alpha\gamma_s k)$. If no such $k$ exists, we set $k_t=0$ and $R_t=\emptyset$.

In practice, a fixed $(\gamma_t)_{t\geq 1}$ schedule can be conservative. To overcome this, we calibrate the schedule using the labeled logs. In particular, for $\lambda$ in an ordered set $\Lambda\subset (0,\infty)$ of candidate scales, we scale the schedule by replacing $\gamma_t$ with $\lambda\gamma_t$. The goal is to select the largest possible $\lambda\in \Lambda$ while maintaining the guarantee of Equation~(\ref{eq:trajectory-fdr}).

Using labeled historical runs, we observe the per-run process loss
\begin{equation*}
  \ell(\lambda)
  \coloneqq
  \max_{t\le\tau_\kappa(\lambda)}\mathrm{FDP}(R_t(\lambda)),
\end{equation*}
whose expectation upper bounds the trajectory FDR up to the $\kappa$-acceptance time on the unobserved run. The goal is to select the largest scale in $\Lambda$ whose historical losses still certify $\esp[\ell_{\mathrm{test}}(\lambda)]\le\alpha$, for the unlabeled deployment run. This loss is non-monotone in $\lambda$, however, so the standard monotone CRC route can fail \citep{angelopoulos2024conformal}. Section~\ref{sec:crc} introduces \emph{robust-core} CRC, a finite-sample method designed for exactly this non-monotone setting.

\section{Related Work}\label{sec:related}

\paragraph{Exchangeability-based $e$-values.} Likelihood ratios provide a foundational tool for statistical testing. Nonetheless, when the underlying distributions are unknown, estimating the likelihoods can compromise finite-sample validity. Universal inference \citep{wasserman2020universal} addresses this through data splitting, fitting the likelihood on one split and evaluating its ratio to the null likelihood on an independent split. Within the testing-by-betting framework \citep{shafer2021testing}, exchangeability-based $e$-values offer a different route by using symmetry under the null to normalize a test statistic. In the group-invariance framework of \citet{koning2023measuring}, this yields statistics of the form $T(x)/\mathbb{E}_G[T(Gx)]$, where the denominator averages over null-preserving transformations. Closest to our local construction, the Besag--Clifford $e$-value of \citet{dombowsky2026besag} normalizes a nonnegative statistic $T$ of data $X$ against $M$ reference samples $Y^{(1)},\ldots,Y^{(M)}$ exchangeable with $X$ under the null, namely $(M+1)T(X)/\{T(X)+\sum_{m=1}^{M}T(Y^{(m)})\}$, and is valid for every $M$. With i.i.d. reference samples from the null and $T$ proportional to the likelihood ratio, it converges almost surely to that ratio as $M\to\infty$. Their main interest is a null known only up to a normalizing constant, where reference samples come from MCMC through the parallel method of \citet{besag1989generalized} and Markov chain mixing controls the loss in expected log-growth. Exchangeability also underlies the sequential Monte Carlo tests of \citet{fischer2024sequential}, who accumulate betting wealth as additional reference samples are generated.

Our setting instead introduces a new hypothesis at each generator round. We construct a marginally valid $e$-value for each proposal by normalizing learned scores against a finite pool of logged null windows, then obtain stopping-time control through online multiple testing. With the true likelihood ratio as the score, this statistic is a Besag--Clifford $e$-value with i.i.d. reference samples. For this local construction, Section~\ref{sec:marginal} separates estimation error from the finite-pool normalization penalty relative to the ideal likelihood ratio. We establish matching finite-sample rates for the normalization penalty under fractional-moment assumptions that cover heavy-tailed likelihood ratios, and show that it vanishes under finite KL divergence alone.

\paragraph{Sequential testing and online FDR control.} Sequential hypothesis testing controls the probability of a false rejection even when the sample size is chosen adaptively or if the data arrives sequentially. In particular, $e$-processes leverage test supermartingales to control the probability of a false rejection under the null \citep{ramdas2023game}. E-valuator \citep{sadhuka2025evaluator} applies this perspective to agent verification, using successive verifier scores to flag trajectories that appear likely to end unsuccessfully. Its error criterion is the probability of flagging a trajectory whose terminal output would actually be correct. In this paper, our aim is to control the expected fraction of incorrect proposals accepted when the generator stops. Each round introduces a new proposal and a corresponding hypothesis, making this an online multiple-testing problem. The offline $e$-BH procedure controls FDR under arbitrary dependence among $e$-values \citep{wang2020false}, and \citet{xu2023online} introduce e-LOND with the same dependence robustness for hypotheses arriving sequentially. In our setting, the stopping time depends on the observed results and we leverage the online $e$-BH procedure of \citet{fischer2024onlineebh}, which guarantees the stronger $\esp[\sup_{t\ge1}\mathrm{FDP}(R_t)]\le\alpha$ result.

\paragraph{Risk control under non-monotone losses.} Conformal risk control (CRC) generalizes the coverage guarantees of conformal prediction to bounds on expected loss \citep{angelopoulos2024conformal}. The original procedure requires the loss to be monotone in the calibrated parameter, and its guarantee can fail when this assumption is violated. For monotone losses, the guarantee is tight up to an $O(1/n)$ term. Learn-then-Test handles non-monotone losses by testing which candidate parameters satisfy the risk constraint and controlling the familywise error rate across these tests \citep{angelopoulos2022learn}. More recently, \citet{angelopoulos2026conformal} replaces monotonicity with algorithmic stability and obtains risk guarantees with additional slack from stability and discretization. Whereas \citet{angelopoulos2026conformal}'s theory is based on leave-one-out stability, \citet{prinster2026conformal} also study a different approach to CRC under replace-one stability, along with risk-controlled policy selection in an agentic-loop setting that differs from ours by only using a generator, and no verifier. Our robust-core CRC method restores monotonicity by lifting calibration to sets of parameter values, yielding finite-sample risk control under exchangeability without requiring algorithmic stability, together with an exact algorithm on finite grids.

\section{Roundwise Marginally Valid \texorpdfstring{$e$}{e}-Values}\label{sec:marginal}

We now proceed to construct valid $e$-values for the roundwise testing problems of Section~\ref{sec:overview}. Intuitively, the construction is a betting game played against the null. We embed the unlabeled window $W_t$ among a set of realizations sampled from the null $P_{0,t}$ and ask a strategy to point to its position. Under the null, $W_t$ is exchangeable with the logged null windows, so no strategy can locate it more reliably. Nonetheless, under the alternative, if $W_t$ follows a different law $P_{1,t}$ from the null sample, it is possible to construct a strategy which identifies its position more often.

Concretely, fix $t$ and consider the simple-versus-simple test $H_{0,t}:W_t\sim P_{0,t}$ against $H_{1,t}:W_t\sim P_{1,t}$. Let $\ncal_t\coloneqq\{i\le n:Y_{t,i}=0\}$ index the logged null windows and set $n_{0,t}\coloneqq|\ncal_t|+1$. Draw a uniformly random ordering of the augmented null batch $\{W_t\}\cup\{W_{t,i}:i\in\ncal_t\}$ and write the resulting indexed version as
\begin{equation*}
  \mathbf W_t^0 \coloneqq \bigl(W_t^{(1)},\ldots,W_t^{(n_{0,t})}\bigr),
\end{equation*}
and let $i^\star\in\{1,\ldots,n_{0,t}\}$ denote the coordinate occupied by $W_t$. Conditional on the realized logged labels, the batch size $n_{0,t}$ is fixed. Under $H_{0,t}$ the unobserved window and the logged null windows are exchangeable draws from $P_{0,t}$, so $\mathbf W_t^0\sim P_{0,t}^{\otimes n_{0,t}}$. The random permutation thus renders the hidden position $i^\star$ uniform on $\{1,\ldots,n_{0,t}\}$ and independent of the ordered batch. Under $H_{1,t}$,  the window $W_t$ is the only element not drawn from the null. Its position $i^\star$ is still uniform on $\{1,\ldots,n_{0,t}\}$, with $W_t^{(i^\star)}\sim P_{1,t}$ and the remaining coordinates drawn from $P_{0,t}$.

A \emph{betting strategy} is a measurable map
\begin{equation*}
  \beta_t(\mathbf W_t^0)
  =
  \bigl(\beta_{t,1}(\mathbf W_t^0),\ldots,\beta_{t,n_{0,t}}(\mathbf W_t^0)\bigr)
  \in [0,1]^{n_{0,t}}
\end{equation*}
that is budgeted, so that $\sum_{i=1}^{n_{0,t}}\beta_{t,i}(\mathbf W_t^0)\le 1$. We restrict to \emph{batch-symmetric} strategies, namely those for which permuting the coordinates of $\mathbf W_t^0$ permutes the output coordinates in the same way. The associated \emph{index-betting $e$-value} is the batch-size-scaled stake placed on the hidden unobserved coordinate,
\begin{equation*}
  E_t^{\beta} (\cdot)\coloneqq n_{0,t}\,\beta_{t,i^\star} (\cdot).
\end{equation*}
In words, the bettor spreads one unit of capital across the $n_{0,t}$ candidate positions and is paid at odds $n_{0,t}$ if the hidden $W_t$ position is guessed correctly. It is straightforward to show that $E_t^{\beta}$ is indeed a valid $e$-value.

\begin{proposition}[Validity of index-betting $e$-values]\label{prop:index-valid}
  For any budgeted, batch-symmetric $\beta_t$, $\esp_{0,t}\bigl[E_t^{\beta}\bigr] \le 1$.
\end{proposition}
\begin{proof}
  Condition on the logged labels, and hence on $n_{0,t}$. Under $H_{0,t}$, the ordered batch is exchangeable, and the random ordering makes $i^\star$ independent of $\mathbf W_t^0$ and uniform on $\{1,\ldots,n_{0,t}\}$. Hence
  \begin{equation*}
    \esp_{0,t}[E_t^\beta\mid \mathbf W_t^0]
    =
    n_{0,t}\,\esp_{0,t}[\beta_{t,i^\star}\mid \mathbf W_t^0]
    =
    \sum_{i=1}^{n_{0,t}}\beta_{t,i}(\mathbf W_t^0)
    \le 1.
  \end{equation*}
\end{proof}

The same hidden-index representation supports a power analysis of the $e$-value, measured by the expected log-growth $\esp_{1,t}\!\left[\log E_t^\beta\right]$ under the alternative. We call a strategy \emph{Kelly-optimal} when it maximizes this log-growth among budgeted, batch-symmetric strategies. Under $H_{1,t}$, power is controlled by how well the strategy identifies the coordinate occupied by $W_t$ in the permuted batch. Proposition~\ref{prop:kelly-lr} shows that the posterior of $i^\star$ given $\mathbf W_t^0$ is proportional to the likelihood ratio, and the Kelly-optimal strategy matches this posterior, mirroring classical log-optimal betting results \citep{kelly1956new,larsson2024numeraire}.

\begin{proposition}[Kelly-optimality of likelihood-ratio betting]\label{prop:kelly-lr}
  Assume $P_{1,t}\ll P_{0,t}$ and write $L_t=\mathrm dP_{1,t}/\mathrm dP_{0,t}$. Under the hidden-index alternative indexed by $H_{1,t}$ above, conditional on the indexed batch $\mathbf W_t^0$,
  \begin{equation*}
    \Pr(i^\star=i\mid \mathbf W_t^0)
    =
    \frac{L_t(W_t^{(i)})}{\sum_{j=1}^{n_{0,t}}L_t(W_t^{(j)})},
    \qquad i=1,\ldots,n_{0,t}.
  \end{equation*}
  Consequently, among all budgeted strategies, and hence among budgeted, batch-symmetric strategies,  $\esp\!\bigl[\log(n_{0,t}\beta_{t,i^\star})\mid \mathbf W_t^0\bigr]$ is uniquely maximized by
  \begin{equation*}
    \beta_{t,i}^{\mathrm{LR}}
    \coloneqq
    \frac{L_t(W_t^{(i)})}{\sum_{j=1}^{n_{0,t}} L_t(W_t^{(j)})},
    \qquad i=1,\ldots,n_{0,t}.
  \end{equation*}
\end{proposition}

The random ordering of $\mathbf{W}_t^0$ is an analytical device. On the original indices, the Kelly-optimal rule is
\begin{equation*}
  E_t^{\mathrm{idx}}
  \coloneqq
  \bigl(|\ncal_t|+1\bigr)
  \frac{L_t(W_t)}{L_t(W_t)+\sum_{i\in\ncal_t}L_t(W_{t,i})}.
\end{equation*}

In practice, $L_t$ is unknown and must be estimated. For each $i\in\ncal_t\cup\{\mathrm{test}\}$, let $\widehat{L}_{t,-i}$ denote the score produced by an order-invariant leave-one-out fitting rule trained on $\calW_t$ with the pair indexed by $i$ removed. The resulting vector of leave-one-out scores is then permutation-equivariant, which is the batch symmetry required above. For validity, it is enough that the betting function be batch-symmetric and budgeted. We can therefore replace $L_t$ with an estimated likelihood ratio and define the $e$-value
\begin{equation}\label{eq:adaptive-evalue-empirical}
  \widehat{E}_t^{\mathrm{idx}}
  \coloneqq
  \bigl(|\ncal_t|+1\bigr)
  \frac{\widehat{L}_{t,-\mathrm{test}}(W_t)}
  {\widehat{L}_{t,-\mathrm{test}}(W_t)
  +\sum_{i\in\ncal_t}\widehat{L}_{t,-i}(W_{t,i})}.
\end{equation}

\subsection{Power decomposition and asymptotic rates}

The index $e$-value $E_t^{\mathrm{idx}}$ is optimal within the index-betting family, yet may still be suboptimal relative to the likelihood ratio $L_t(W_t)$. We adopt it because $L_t$ is unknown in practice, and the index-betting construction turns empirical proxies into valid $e$-values as in Equation~(\ref{eq:adaptive-evalue-empirical}). To characterize the log-optimality gap between the empirical index $e$-value $\widehat E_t^{\mathrm{idx}}$ and the likelihood ratio, we decompose it into a betting term and an estimation term,
\begin{equation}\label{eq:estimation-decomposition}
  \esp_{1,t}\!\left[\log L_t(W_t)-\log \widehat E_t^{\mathrm{idx}}\right]
  =
  \underbrace{\esp_{1,t}\!\left[\log L_t(W_t)-\log E_t^{\mathrm{idx}}\right]}_{\text{betting cost}}
  +
  \underbrace{\esp_{1,t}\!\left[\log E_t^{\mathrm{idx}}-\log \widehat E_t^{\mathrm{idx}}\right]}_{\text{estimation residual}}.
\end{equation}
The \emph{betting cost} is the fundamental price of the index-betting construction, the shortfall from $L_t(W_t)$ to the best index-betting $e$-value $E_t^{\mathrm{idx}}$, paid even when $L_t$ is known exactly. The \emph{estimation residual} is the price of ignorance, the further gap from $E_t^{\mathrm{idx}}$ to $\widehat E_t^{\mathrm{idx}}$. The theorem below controls it by the worst log-ratio error on the augmented batch.

\begin{theorem}[Estimation residual]\label{thm:power-decomp}
  Assume $L_t>0$ $P_{0,t}$-a.s. and that $\widehat{L}_{t,-i}(W_{t,i})>0$ for every $i\in \ncal_t\cup\{\mathrm{test}\}$ a.s. Then
  \begin{equation*}
    \left|\log E_t^{\mathrm{idx}}-\log \widehat{E}_t^{\mathrm{idx}}\right|
    \le
    2\max_{i\in \ncal_t\cup\{\mathrm{test}\}}
    \left|\log L_t(W_{t,i})-\log \widehat{L}_{t,-i}(W_{t,i})\right|.
  \end{equation*}
  Consequently, if that worst-case error, $\max_{i\in \ncal_t\cup\{\mathrm{test}\}}\lvert\log L_t(W_{t,i})-\log \widehat{L}_{t,-i}(W_{t,i})\rvert$, has finite expectation under the alternative, then
  \begin{equation*}
    \left|
    \esp_{1,t}\left[\log E_t^{\mathrm{idx}}
    -\log \widehat{E}_t^{\mathrm{idx}}\right]
    \right|
    \le
    2\esp_{1,t}\left[
      \max_{i\in \ncal_t\cup\{\mathrm{test}\}}
      \left|\log L_t(W_{t,i})-\log \widehat{L}_{t,-i}(W_{t,i})\right|
    \right].
  \end{equation*}
\end{theorem}

Thus, convergence of the expected worst-case log-ratio error implies convergence of the estimation residual. The remaining question is how fast the betting cost vanishes with the batch size $n_{0,t}$. We show that a bounded fractional moment of $L_t$ under $P_{0,t}$ is enough for the betting cost to vanish in $n_{0,t}$, at a rate controlled by the moment exponent.

\begin{theorem}[Betting cost upper bound]\label{thm:trunc-power-gap}
  Condition on the realized logged labels, so that $n_{0,t}$ is fixed. Suppose
  $\esp_{P_{0,t}}[L_t^{1+\delta}]<\infty$ for some $\delta\in(0,1]$, and set
  \begin{equation*}
    C_\delta\coloneqq\sup_{x>0}\frac{\log(1+x)}{x^\delta},
  \end{equation*}
  which is finite for every $\delta\in(0,1]$. Then
  \begin{equation*}
    \esp_{1,t}\!\left[\log L_t(W_t)-\log E_t^{\mathrm{idx}}\right]
    \le
    \frac{C_\delta\,\esp_{P_{0,t}}\!\left[L_t\,(L_t-1)_+^\delta\right]}{n_{0,t}^\delta}
    \le
    \frac{C_\delta\,\esp_{P_{0,t}}\!\left[L_t^{1+\delta}\right]}{n_{0,t}^\delta}.
  \end{equation*}
\end{theorem}

Since $C_1=1$, setting $\delta=1$ recovers the classical $O(1/n_{0,t})$ rate, and $\delta<1$ extends the bound into the heavy-tailed regime where $\esp_{P_{0,t}}[L_t^2]=\infty$. For every $\delta\in(0,1)$, this polynomial rate is tight over the moment class $ \mathcal P_{1+\delta}(B)  =\{(P_{0,t},P_{1,t}):\esp_{P_{0,t}}[L_t^{1+\delta}]\le B\}$.

\begin{theorem}[Matching lower bound]\label{thm:lower-bound}
  For every $\delta\in(0,1)$ and $B > 1$, there exists $c = c(\delta, B)>0$ such that for all $n_{0,t}$ sufficiently large some pair $(P_{0,t},P_{1,t}) \in \mathcal{P}_{1+\delta}(B)$ satisfies
  \begin{equation*}
    \esp_{1,t}\!\Big[\log L_t(W_t) - \log E_t^{\mathrm{idx}}\Big]
    \ge
    c\, n_{0,t}^{-\delta}.
  \end{equation*}
  Since $E_t^{\mathrm{idx}}$ is Kelly-optimal by Proposition~\ref{prop:kelly-lr}, this worst-case rate also lower bounds every budgeted batch-symmetric strategy.
\end{theorem}

Finally, finite KL divergence alone is enough for the betting cost to vanish asymptotically.

\begin{corollary}[KL-only convergence]\label{cor:kl-dominance}
  If $\KL(P_{1,t}\|P_{0,t})<\infty$, then $\esp_{1,t}[\log L_t(W_t)-\log E_t^{\mathrm{idx}}]\to 0$ as $n_{0,t}\to\infty$.
\end{corollary}

\subsection{Examples and simulations}\label{subsec:sec4-examples-simulations}

We instantiate the local testing problem on two one-dimensional families chosen to separate light- and heavy-tailed likelihood ratios. The light-tailed example sets $P=\mathcal N(0,1)$ against $Q=\mathcal N(2.5,1)$. The heavy-tailed example tests $P=\mathrm{Pareto}(4)$ against $Q=\mathrm{Pareto}(1)$ on $x\ge1$, where $\mathrm{Pareto}(a)$ has density $f_a(x)=a x^{-(a+1)}$. We first illustrate the likelihood-ratio behavior in these two settings.

\begin{example}[Gaussian shift]
  A Gaussian shift sits in the light-tailed regime. If $W_t\sim\mathcal N(0,\Sigma)$ under $P_{0,t}$ and $W_t\sim\mathcal N(\mu,\Sigma)$ under $P_{1,t}$, the log-likelihood ratio $\log L_t(W_t)$ is Gaussian with mean $-\rho^2/2$ and variance $\rho^2$, where $\rho^2=\mu^\top\Sigma^{-1}\mu$. Hence $\esp_{P_{0,t}}[L_t^k]=e^{k(k-1)\rho^2/2}$ for every $k$, so all moments are finite. Theorem~\ref{thm:trunc-power-gap} with $\delta=1$ therefore recovers a $1/n_{0,t}$ rate.
\end{example}

\begin{example}[Pareto pair] Let $f_0(x)=ax^{-(a+1)}$ and $f_1(x)=bx^{-(b+1)}$ on $x\ge 1$ with $a>b>0$, so $L_t(x)=(b/a)x^{a-b}$. Here $\esp_{P_{0,t}}[L_t^2]<\infty$ only when $a<2b$, whereas $\esp_{P_{0,t}}[L_t^{1+\delta}]<\infty$ iff $\delta<b/(a-b)$, which holds for some $\delta>0$ for $a>b$.
\end{example}

\begin{figure}[htbp]
  \centering
  \includegraphics[width=\linewidth]{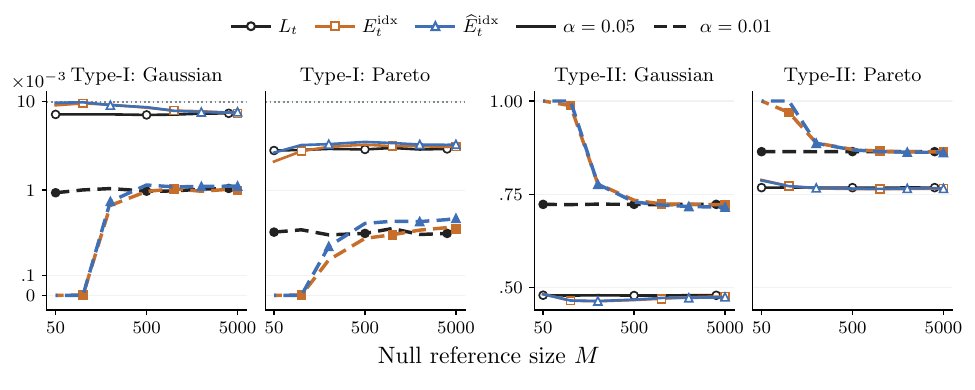}
  \caption{Type-I and Type-II error rates across reference sizes $M$ for $L_t$, $E_t^{\mathrm{idx}}$, and $\widehat E_t^{\mathrm{idx}}$. Solid curves with open markers correspond to $\alpha=0.05$ and dashed curves with filled markers to $\alpha=0.01$.}
  \label{fig:sec4-error-curves}
\end{figure}

For each family, we sweep the reference set size $M\in\{50,100,200,500,1000,2000,4000,5000\}$ and report results at nominal levels $\alpha\in\{0.05,0.01\}$, where $M=n_{0,t}-1$. For the Gaussian experiments, we use $P_{0,t}=\mathcal{N}(0,1)$ and $P_{1,t}=\mathcal{N}(2.5,1)$, while for the Pareto experiments we use $P_{0,t}=\mathrm{Pareto}(4)$ and $P_{1,t}=\mathrm{Pareto}(1)$. To estimate the likelihood ratio, we fit standardized logistic regression on 100 null and 100 alternative observations, using feature $z$ for the Gaussian family and $\log z$ for the Pareto family. The fitted odds are clipped to $[10^{-3},10^3]$. Figure~\ref{fig:sec4-error-curves} reports the Type-I and Type-II error rates of three $e$-values, namely the true likelihood ratio $L_t$, index betting with the true likelihood ratio $E_t^{\mathrm{idx}}$, and index betting with the estimated likelihood ratio $\widehat{E}_t^{\mathrm{idx}}$.

In the second experiment, we aim to characterize how the tail of the likelihood ratio controls the finite-reference penalty. To remove estimator error, Figure~\ref{fig:sec4-delta-sweep} compares only the likelihood ratio with $E_t^{\mathrm{idx}}$. The sweep is organized by the largest finite fractional moment
\begin{equation*}
  \delta
  \coloneqq
  \sup\{\gamma>0:\esp_{P_{0,t}}[L_t^{1+\gamma}]<\infty\}.
\end{equation*}
We sweep $\delta$ over $\{0.025,\ldots,1\}$ by choosing a null--alternative pair whose likelihood ratio has exactly that fractional-moment exponent. In the Gaussian case, we use $P=\mathcal N(0,1)$ against $Q=\mathcal N(0,1+1/\delta)$, while in the Pareto case we use $P=\mathrm{Pareto}(4)$ against $Q=\mathrm{Pareto}(b)$ with $b=4\delta/(1+\delta)$. Power is reported on a normalized log-growth scale, $\esp_Q[\log E]/\KL(Q\|P)$. Since $\esp_Q[\log L_t]=\KL(Q\|P)$, the likelihood ratio sits at exactly one, and a curve at $0.7$ means the index e-value retains $70\%$ of the ideal log-growth. The resulting curves follow both directions predicted by Theorem~\ref{thm:trunc-power-gap}. The gap closes as $\delta$ grows and as the reference batch $M$ grows.
\begin{figure}[htbp]
  \centering
  \includegraphics[width=\linewidth]{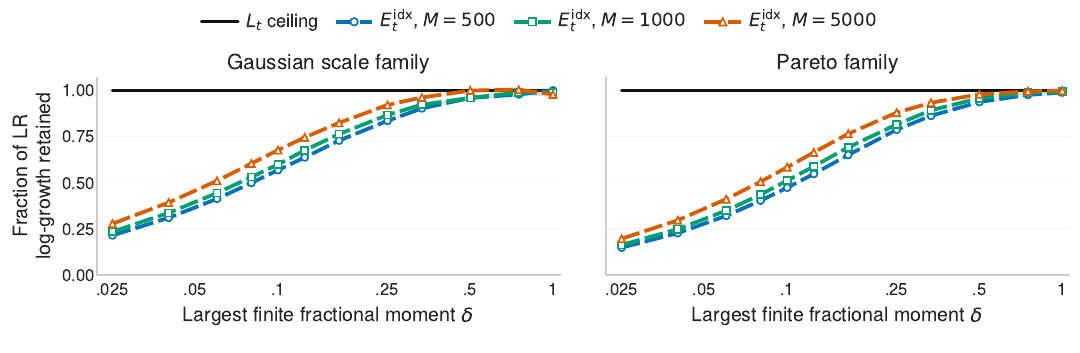}
  \caption{Normalized log-growth $\esp_Q[\log E]/\KL(Q\|P)$ of the oracle
    index e-value $E_t^{\mathrm{idx}}$ as a function of the
    fractional-moment exponent $\delta$, for reference sizes
    $M\in\{500,1000,5000\}$. Each point averages $2\times 10^5$ repetitions.}
  \label{fig:sec4-delta-sweep}
\end{figure}

\section{Robust-Core Risk Control for Non-Monotone Losses}\label{sec:crc} Using exchangeable labeled samples, CRC \citep{angelopoulos2024conformal} calibrates a loss parameter so that the test-time loss has controlled expectation. Let $(\Lambda,\preceq)$ be a totally ordered parameter space, where $\lambda\preceq\lambda'$ means that $\lambda'$ is preferred by the downstream selection criterion. For the online $e$-BH family introduced in Section~\ref{sec:overview}, larger values of $\lambda$ stop the loop earlier, save computational costs, and are therefore preferred. Throughout this section we treat CRC abstractly, then specialize it for our setting in Section~\ref{subsec:robust-core-calibration}.

We assume that $\ell_1,\ldots,\ell_{n+1}$ are exchangeable random loss profiles, where each $\ell_i:\Lambda\to[0,1]$ is a measurable function of $\lambda$ and auxiliary randomness that typically represents the underlying data, which we leave implicit to keep the notation light. In our setting, only $\ell_1,\ldots,\ell_n$ are observed and $\ell_{n+1}$ is hidden.
For $A\subseteq\Lambda$, write $\inf_{\preceq}A$, $\sup_{\preceq}A$, and $\max_{\preceq}A$ for the infimum, supremum, and maximum with respect to $\preceq$. Whenever a rule below selects a supremum, we assume that it lies in the feasible set. Given a level $\alpha\in(0,1)$, a CRC method chooses an $(\ell_1,\ldots,\ell_n)$-dependent $\widehat\lambda\in\Lambda$, which satisfies
\begin{equation}\label{eq:crc-risk-target}
  \esp[\ell_{n+1}(\widehat\lambda)] \le \alpha.
\end{equation}

When each loss profile is non-decreasing along $\preceq$, i.e., $\lambda\preceq\lambda'$ implies $\ell_i(\lambda)\le\ell_i(\lambda')$ for every $i$, the classical CRC approach computes
\begin{equation}\label{eq:crc-monotone}
  \widehat\lambda_{\mathrm{CRC}}
  \coloneqq
  \sup_{\preceq}\!\left\{\lambda \in \Lambda \,\Big|\,
  \frac{1}{n+1}\!\left(\sum_{i=1}^{n}\ell_i(\lambda) + 1\right) \le \alpha\right\},
\end{equation}
which satisfies the guarantee in Equation~(\ref{eq:crc-risk-target}) \citep{angelopoulos2024conformal}. When monotonicity fails, the above rule can be invalid \citep{angelopoulos2024conformal}. A conservative alternative replaces each $\ell_i$ by its running supremum $\bar\ell_i(\lambda)\coloneqq\sup_{\lambda'\preceq\lambda}\ell_i(\lambda')$ and applies the classical rule to these envelopes,
\begin{equation}\label{eq:envelope-crc}
  \widehat\lambda_{\mathrm{env}}
  \coloneqq
  \sup_{\preceq}\!\left\{\lambda \in \Lambda \,\Big|\,
  \frac{1}{n+1}\!\left(\sum_{i=1}^{n}\bar\ell_i(\lambda) + 1\right) \le \alpha\right\}.
\end{equation}
The envelopes are non-decreasing and dominate the original curves, $\ell_i(\lambda)\le\bar\ell_i(\lambda)$, so the monotone guarantee applied to $\bar\ell_i$ gives
\begin{equation*}
  \esp[\ell_{n+1}(\widehat\lambda_{\mathrm{env}})]
  \le
  \esp[\bar\ell_{n+1}(\widehat\lambda_{\mathrm{env}})]
  \le \alpha,
\end{equation*}
satisfying Equation~(\ref{eq:crc-risk-target}). Nonetheless, the above envelope approach can be conservative as a local spike is treated as permanent. Hence parameters that are safe on the realized curve can be discarded.

\subsection{Robust core on ordered spaces}\label{subsec:robust-core}

Robust-core CRC handles non-monotonicity by lifting the problem and calibrating a set in $2^\Lambda$. For a set $C\subseteq\Lambda$, define $L_i(C)\coloneqq\sup_{\lambda\in C}\ell_i(\lambda)$, with $L_i(\emptyset)=0$. This set loss is monotone under inclusion. If $C\subseteq C'$, then $L_i(C)\le L_i(C')$. Thus, even when the scalar curve $\lambda\mapsto\ell_i(\lambda)$ is non-monotone, the lifted set-valued problem has a monotone structure. We assume the relevant suprema exist. This condition holds automatically in the finite-grid setting introduced in Section~\ref{subsec:finite-grid}.

Let $\mathcal L\subseteq[0,1]^\Lambda$ be an admissible class containing all possible loss functions. A complete-data certificate is a symmetric map $G:\mathcal L^{n+1}\to 2^{\Lambda}$ whose output has controlled empirical set-loss for every complete tuple $\ell_{1:n+1}\in\mathcal L^{n+1}$:
\begin{equation}\label{eq:abstract-feasibility}
  \frac{1}{n+1}\sum_{i=1}^{n+1}L_i\bigl(G(\ell_{1:n+1})\bigr) \le \alpha.
\end{equation}
Symmetry means that $G$ depends only on the unordered collection of loss curves. Concretely, for every permutation $\sigma$ of $\{1,\ldots,n+1\}$,
\begin{equation*}
  G(\ell_1,\ldots,\ell_{n+1})
  =
  G(\ell_{\sigma(1)},\ldots,\ell_{\sigma(n+1)}).
\end{equation*}

If the unobserved loss profile $\ell_{n+1}$ were observed, any parameter in $G(\ell_{1:n+1})$ would be safe by the original conformal-risk-control exchangeability argument. Indeed, condition on the unordered multiset $\{\ell_1,\ldots,\ell_{n+1}\}$. The index of the unobserved profile is then uniform on $\{1,\ldots,n+1\}$, while every $\lambda\in G(\ell_{1:n+1})$ is measurable with respect to this unordered multiset by symmetry. Hence, for any such $\lambda$,
\begin{equation*}
  \mathbb E\!\left[\ell_{n+1}(\lambda)
  \mid \{\ell_1,\ldots,\ell_{n+1}\}\right]
  =
  \frac{1}{n+1}\sum_{i=1}^{n+1}\ell_i(\lambda)
  \le
  \frac{1}{n+1}\sum_{i=1}^{n+1}
  L_i\bigl(G(\ell_{1:n+1})\bigr)
  \le \alpha.
\end{equation*}
One canonical way to construct such a $G$ is the greedy selector in Algorithm~\ref{alg:greedy-selector}, which scans $\Lambda$ from most to least preferred and admits a parameter when its inclusion preserves Equation~\eqref{eq:abstract-feasibility}. The rule is permutation invariant in $\ell_1,\ldots,\ell_{n+1}$ and satisfies Equation~\eqref{eq:abstract-feasibility} by construction.

\begin{algorithm}[t]
  \caption{Greedy Selector}\label{alg:greedy-selector}
  \begin{algorithmic}[1]
    \STATE \textbf{Input.} Ordered parameter space $(\Lambda,\preceq)$, complete loss profiles $\ell_1,\ldots,\ell_{n+1}\in\mathcal L$, level $\alpha$.
    \STATE \textbf{Initialize.} $C \gets \varnothing$, $R \gets \Lambda$.
    \WHILE{$R\neq\varnothing$}
    \STATE
    \begin{equation*}
      R \gets \left\{\lambda\in\Lambda\setminus C:
      \frac{1}{n+1}\sum_{i=1}^{n+1}L_i(C\cup\{\lambda\})\leq \alpha\right\}
    \end{equation*}
    \IF{$R\neq \varnothing$}
    \STATE $\lambda^\star\gets \max_{\preceq} R$, \quad $C \gets C\cup\{\lambda^\star\}$.
    \ENDIF
    \ENDWHILE
    \RETURN $C$.
  \end{algorithmic}
\end{algorithm}

Nonetheless, $\ell_{n+1}$ is not observed and $G(\ell_1,\ldots,\ell_{n+1})$ cannot be evaluated directly. Instead, consider the parameters that would be certified no matter which admissible $\ell_{n+1}$ is realized:
\begin{equation*}
  C_n\coloneqq\bigcap_{w\in\mathcal L}G(\ell_{1:n},w),
\end{equation*}
where $w$ is a placeholder for any possible realization of $\ell_{n+1}$. Since the realized $\ell_{n+1}$ belongs to $\mathcal L$, $\lambda\in C_n$ implies $\lambda\in G(\ell_1,\ldots,\ell_{n+1})$. Following \citet{angelopoulos2026conformal}, we also assume a fallback parameter $\lambda_{\mathrm{safe}}\in\Lambda$ satisfying
\begin{equation}\label{eq:safe-choice}
  \ell_i(\lambda_{\mathrm{safe}})=0,
  \qquad i=1,\ldots,n+1,
\end{equation}
almost surely. The robust-core estimator chooses the most preferred parameter in $C_n$ when this set is nonempty and falls back otherwise:
\begin{equation}\label{eq:robust-core-estimator}
  \widehat\lambda_{\mathrm{core}} \coloneqq
  \begin{cases}
    \max_{\preceq} C_n, & C_n\neq\varnothing, \\
    \lambda_{\mathrm{safe}}, & C_n=\varnothing.
  \end{cases}
\end{equation}

\begin{theorem}[Validity of the robust core]\label{thm:tree-crc-valid}
  Under the setup above, suppose the loss profiles are exchangeable, $G$ is symmetric and satisfies Equation~(\ref{eq:abstract-feasibility}), and Equation~(\ref{eq:safe-choice}) holds. Then $\esp[\ell_{n+1}(\widehat\lambda_{\mathrm{core}})]\le \alpha$.
\end{theorem}

Nonetheless, $C_n$ is not generally computationally tractable as it requires the intersection of $G(\ell_{1:n},w)$ over all admissible missing curves $w$. We now specialize to a finite grid and construct $C_n$ with $G$ fixed as the greedy selector in Algorithm~\ref{alg:greedy-selector}.

\subsection{Finite-grid specialization}\label{subsec:finite-grid}

Consider a finite grid $\Lambda=\{\lambda_1,\ldots,\lambda_m\}\subset\mathbb R$ listed from least to most preferred, so $\lambda_1\preceq\cdots\preceq\lambda_m$. On a grid, the greedy certificate of Algorithm~\ref{alg:greedy-selector} reduces to a single scan from $\lambda_m$ down to $\lambda_1$. When the scan visits $\lambda_k$, it includes $\lambda_k$ in the running set $C$ exactly when
\begin{equation*}
  \text{include } \lambda_k \text{ in } C
  \quad\Longleftrightarrow\quad
  \frac{1}{n+1}\sum_{i=1}^{n+1}
  \max_{\lambda\in C\cup\{\lambda_k\}}\ell_i(\lambda)\le \alpha.
\end{equation*}
The scan reproduces Algorithm~\ref{alg:greedy-selector} because the coordinatewise maxima are monotone in $C$: a point that fails when first visited cannot become feasible after $C$ has grown. At calibration time, however, $\ell_{n+1}$ is not observed. We must compute the intersection over all possible completions $w:\Lambda\to[0,1]$, where $w(\lambda_k)$ is the loss the unobserved run would incur at $\lambda_k$. Letting $\mathcal U=[0,1]^\Lambda$ denote the unconstrained completion class, the robust core specializes to
\begin{equation*}
  C_n=\bigcap_{w\in\mathcal U}G(\ell_{1:n},w),
\end{equation*}
where $G(\ell_{1:n},w)$ denotes the scan applied to the augmented tuple $(\ell_1,\ldots,\ell_n,w)$. A grid point lies in $C_n$ when no admissible completion can dislodge it. Since $\mathcal U$ is infinite, $C_n$ cannot be obtained by enumerating completions. Algorithm~\ref{alg:crc-tree} instead tracks, for each scan branch, a finite sufficient statistic of its consistent possibilities.

\paragraph{Algorithm description.} Before processing $\lambda_k$, the include or skip decision depends on the already selected set $C$ only through the calibration maxima $M_i=\max_{\lambda\in C}\ell_i(\lambda)$ and the held-out maximum $u=\max_{\lambda\in C}w(\lambda)$. The vector $M$ is pinned down by $C$ and the observed losses, while $u$ depends on the unobserved completion $w$. Rather than carry all completions consistent with a branch, the algorithm stores only their projection onto $u$, which remains an interval $I=[a,b]\subseteq[0,1]$ throughout the recursion. A frontier state is therefore a pair $(M,I)$, initialized at $(0,[0,0])$ before the scan begins. After deduplication, one state may represent branches with different selected sets, but their future decisions coincide because they share the same sufficient pair $(M,I)$.

At grid point $\lambda_k$, write $M'_i=\max\{M_i,\ell_i(\lambda_k)\}$ and $\Theta=(n+1)\alpha-\sum_{i=1}^n M'_i$. For $u\in[a,b]$ and $v=w(\lambda_k)\in[0,1]$, the scan includes $\lambda_k$ iff
\begin{equation*}
  \frac{1}{n+1}\left(\sum_{i=1}^n M'_i+\max(u,v)\right)\le \alpha,
  \qquad\text{equivalently}\qquad
  \max(u,v)\le \Theta.
\end{equation*}
Each state branches into at most two successors. The include branch is reachable when $a\le\Theta$; then $M$ becomes $M'$ and the held-out interval becomes $[a,\min\{1,\Theta\}]$. The skip branch is reachable when $\Theta<1$, since choosing $w(\lambda_k)>\Theta$ skips inclusion while leaving $(M,I)$ fixed.

The scan moves from most to least preferred, expanding the frontier and merging coincident states. It returns the first grid point $\lambda_k$ at which the include branch is reachable from every frontier state and the skip branch from none. Then every admissible completion forces inclusion of $\lambda_k$, so $\lambda_k\in C_n$ as the grid is scanned in order of decreasing preference.

\begin{algorithm}[t]
  \caption{Robust Core Tree-CRC}\label{alg:crc-tree}
  \begin{algorithmic}[1]
    \STATE \textbf{Input.} Calibration losses $l_i(k)=\ell_i(\lambda_k)$, grid $\lambda_1\preceq\cdots\preceq\lambda_m$, level $\alpha$, fallback $\lambda_{\mathrm{safe}}$.
    \STATE \textbf{Initialize.} Frontier $\mathcal F \gets \{(M=0^n,I=[0,0])\}$.
    \FOR{$k=m,m-1,\ldots,1$}
    \STATE $\mathcal F' \gets \varnothing$, $\mathrm{AllInclude} \gets \mathrm{true}$.
    \FOR{each state $(M,I=[a,b])\in\mathcal F$}
    \STATE $M' \gets \max(M,l_{1:n}(k))$ coordinatewise.
    \STATE $\Theta \gets (n+1)\alpha - \sum_{i=1}^n M_i'$.
    \IF{$a\le\Theta$}
    \STATE Add $(M',[a,\min\{1,\Theta\}])$ to $\mathcal F'$.
    \ENDIF
    \IF{$\Theta<1$}
    \STATE Add $(M,[a,b])$ to $\mathcal F'$, set $\mathrm{AllInclude} \gets \mathrm{false}$.
    \ENDIF
    \ENDFOR
    \STATE Deduplicate repeated states in $\mathcal F'$, set $\mathcal F \gets \mathcal F'$.
    \IF{$\mathrm{AllInclude}$}
    \RETURN $\lambda_k$.
    \ENDIF
    \ENDFOR
    \RETURN $\lambda_{\mathrm{safe}}$.
  \end{algorithmic}
\end{algorithm}

\begin{lemma}[Exactness of Robust Core Tree-CRC]\label{lem:crc-tree-correct}
  Let $\mathcal U=[0,1]^\Lambda$, and let $G(\ell_{1:n},w)$ be the set returned by the complete-data greedy scan when the missing profile is $w\in\mathcal U$. Algorithm~\ref{alg:crc-tree} returns
  \begin{equation*}
    \max_{\preceq}\bigcap_{w\in\mathcal U}G(\ell_{1:n},w)
  \end{equation*}
  when this intersection is nonempty, and returns $\lambda_{\mathrm{safe}}$ otherwise.
\end{lemma}

Theorem~\ref{thm:tree-crc-valid} certifies that the robust core controls risk, and Lemma~\ref{lem:crc-tree-correct} certifies that Algorithm~\ref{alg:crc-tree} computes it on the grid. Together they yield a finite-sample calibration rule for non-monotone loss profiles.

\subsection{Non-monotone CRC experiment}

We illustrate robust-core calibration on clipped-Gaussian non-monotone losses over the ordered grid $\Lambda=\{\lambda_k=k/100:k=1,\ldots,100\}$. For example $i$ and grid point $\lambda_k$, the loss is
\begin{equation*}
  \ell_i(\lambda_k)=\mathrm{clip}(\lambda_k+\sigma Z_{i,k},0,1),
  \qquad Z_{i,k}\stackrel{\mathrm{ind}}{\sim}\mathcal N(0,1),
\end{equation*}
independent across examples and grid points. We separately define $\ell_i(0)=0$ for every run and set the fallback $\lambda=0$. We sweep over noise levels $\sigma\in\{0.20,0.10,0.05\}$ and calibration sizes $n\in\{100,200,\ldots,1000,1200,1500,2000,2500,3000\}$, at level $\alpha=0.1$. Larger $\lambda$ values are more preferred and more aggressive. We compare robust-core CRC, the monotone-envelope rule of Equation~(\ref{eq:envelope-crc}), and the finite-grid stable-CRC baseline of \citet{angelopoulos2026conformal}. Figure~\ref{fig:crc-gaussian} plots the mean selected parameter with 95\% Monte Carlo intervals over 100 trials. The dotted line is the highest grid point with population risk less than $\alpha$.

\begin{figure}[htbp]
  \centering
  \includegraphics[width=\linewidth]{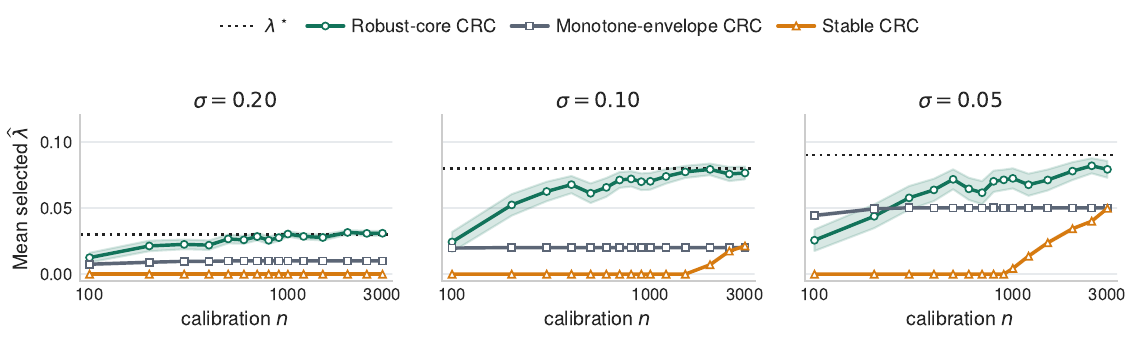}
  \caption{Mean selected parameter on the clipped-Gaussian grid,
    with 95\% Monte Carlo intervals. The dotted line
    marks the highest fixed safe parameter
    $\lambda^\star=\max_{\preceq}\{\lambda_k:\esp[\ell_i(\lambda_k)]\le\alpha\}$.   The three panels show the three clipped-Gaussian noise levels in the sweep.}
  \label{fig:crc-gaussian}
\end{figure}

All three rules control risk, so the comparison is about the aggressiveness with which each rule is able to select the parameter $\widehat\lambda$. Robust-core CRC assumes only exchangeability and boundedness, and its selected parameter still climbs with $n$ toward the fixed safe parameter $\lambda^\star$. The two baselines stay well below it for different reasons. The monotone-envelope rule replaces every loss profile by its running maximum, so a single upward fluctuation at a small grid point suppresses all larger candidates. The stable-CRC baseline retains the raw profiles but subtracts an explicit concentration slack, which shrinks slowly with $n$. The robust core avoids both penalties and approaches $\lambda^\star$.

An astute reader might have noticed that Algorithm~\ref{alg:crc-tree} is exponential in the grid size $|\Lambda|$ in the worst case, since the frontier can hold up to $2^{|\Lambda|}$ states. Nonetheless, we note that the scan terminates as soon as it reaches a grid point that is included from every frontier state, and that states that coincide are merged. In practice, we found the rule to be efficient. Figure~\ref{fig:crc-runtime} reports its selection time on the clipped-Gaussian losses above, for grid sizes up to $|\Lambda|=250$.

\begin{figure}[htbp]
  \centering
  \includegraphics[width=\linewidth]{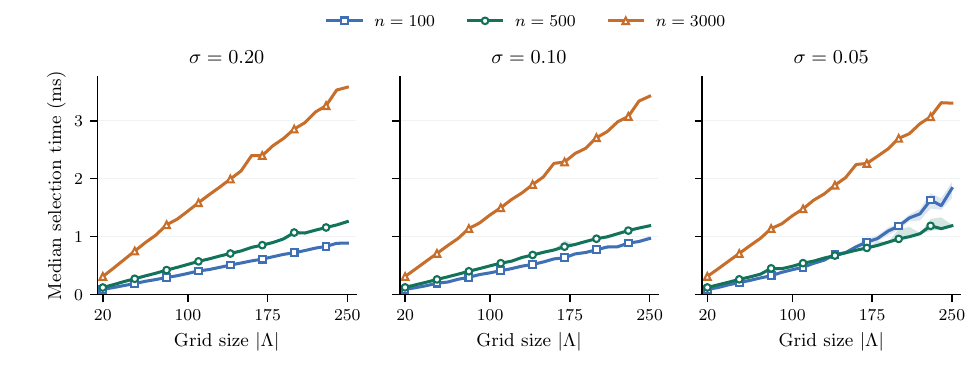}
  \caption{\textbf{Selection time of robust-core CRC against the grid size.} Each panel corresponds to one noise level $\sigma$ of the clipped-Gaussian losses and shows the median selection time in milliseconds against the grid size $|\Lambda|$, at level $\alpha=0.1$, with one curve per calibration size $n$. Curves show the median over 100 independent calibration datasets of each dataset's median time over five repetitions, and shaded bands show pointwise 95\% bootstrap intervals for this median.}
  \label{fig:crc-runtime}
\end{figure}

\section{End-to-End Valid Stoppage}\label{sec:fdr}

We now proceed to introduce the full procedure, which combines the preceding $e$-value construction and CRC results into a single pipeline. To recap our setting and notation, we consider an unlabeled deployment run $\testrun = ((X_t,V_t,S_t,Y_t))_{t=1}^{T}$ whose correctness labels $Y_1,\ldots,Y_T$ are unavailable. At each round $t$, we observe the score window $W_t\coloneqq(S_{t-K},\ldots,S_t)\in\mathcal{S}^{K+1}$ formed from the most recent $K{+}1$ verifier scores and return an accepted set $R_t\subseteq\{1,\ldots,t\}$. These sets are nested, with $R_1\subseteq R_2\subseteq\cdots\subseteq R_T$. Our goal is to ensure that $\esp[\mathrm{FDP}(R_\tau)]\le\alpha$ for every stopping time $\tau\le\tau_\kappa$, as formalized in Equation~(\ref{eq:trajectory-fdr}). Throughout this section, we make use of two different splits of labelled data, the first for constructing the roundwise $e$-value and the second for calibrating the online $e$-BH procedure to make it less conservative. We denote the dataset used in constructing the roundwise $e$-values by 
\begin{equation*}
  \bcal^{\es}
  =
  \{\run^{\es}_1,\ldots,\run^{\es}_{n_{\es}}\},
\end{equation*}
and use the superscript $\es$ to mark quantities belonging to runs in $\bcal^\es$, for example $S^{\es}_{t,j}$ and $Y^{\es}_{t,j}$. We assume that the runs in $\bcal^\es$ and $\testrun$ are independent and identically distributed. 

\subsection{Valid stoppage on the unobserved run}\label{subsec:uncalibrated-rule}

At round $t$, we use $W_t$ to test $H_{0,t}:W_t\sim P_{0,t}$ against $H_{1,t}:W_t\sim P_{1,t}$, where $P_{i,t}$ is the distribution of $W_t|Y_t=i$ for $i\in\{0,1\}$. We construct the corresponding $e$-value using the leave-one-out index-betting procedure of Section~\ref{sec:marginal}. Mirroring Section~\ref{sec:overview}, define
\begin{equation*}
  \calW_t^{\es}
  \coloneqq
  \{(W_t,0)\}
  \cup
  \bigl\{(W^{\es}_{t,j},Y^{\es}_{t,j}):j\le n_{\es}\bigr\} \;\;\; \text{ and }  \;\;\; \ncal_t^{\es} \coloneqq \{j\le n_{\es}:Y^{\es}_{t,j}=0\}.
\end{equation*}
For each $a\in\ncal_t^{\es}\cup\{\mathrm{test}\}$, let $\widehat{L}_{t,-a}$ be the likelihood-ratio estimator fitted to $\calW_t^{\es}$ with the sample indexed by $a$ removed, using an order-invariant fitting rule. The resulting roundwise $e$-value is
\begin{equation}\label{eq:split-evalue}
  \widehat{E}_t
  =
  \bigl(|\ncal_t^{\es}|+1\bigr)
  \frac{\widehat{L}_{t,-\mathrm{test}}(W_t)}
  {\widehat{L}_{t,-\mathrm{test}}(W_t)
  + \sum_{j\in\ncal_t^{\es}}
  \widehat{L}_{t,-j}(W^{\es}_{t,j})}.
\end{equation}
Under the round-$t$ null, Proposition~\ref{prop:index-valid} gives $\esp_{0,t}[\widehat{E}_t]\le 1$.

Using the online $e$-BH procedure, testing using the sequence of $e$-values $(\widehat{E}_t)_{t=1}^{T}$ yields nested accepted sets $R_t(\lambda)$ and the $\kappa$-acceptance time $\tau_\kappa(\lambda)=\inf\{t:|R_t(\lambda)|\ge\kappa\}\wedge T$. Explicitly, for a nonnegative unit schedule $(\gamma_t)_{t\ge1}$ with $\sum_{t\ge1}\gamma_t=1$ and a scale $\lambda\ge0$,
\begin{equation}\label{eq:ebh-rule}
  R_t(\lambda)
  =
  \left\{s\le t:
    \widehat{E}_s
    \ge \frac{1}{\alpha\lambda\gamma_s k_t(\lambda)}
  \right\},
  \qquad
  k_t(\lambda)
  =
  \max\left\{k:
    \sum_{s=1}^{t}
    \one\left\{\widehat{E}_s
    \ge \frac{1}{\alpha\lambda\gamma_s k}\right\}\ge k
  \right\},
\end{equation}
where $\max(\emptyset)=0$ and $1/0=\infty$. At $\lambda=1$, marginal validity of the roundwise $e$-values and the online $e$-BH guarantee of \citet{fischer2024onlineebh} are sufficient to yield
\begin{equation*}
  \esp\Bigl[\sup_{t\ge 1}\mathrm{FDP}(R_t(1))\Bigr]\le\alpha
\end{equation*}
under arbitrary within-run dependence, and hence the trajectory guarantee of Equation~(\ref{eq:trajectory-fdr}). 

\subsection{Robust-core calibration}\label{subsec:robust-core-calibration}
While setting $\lambda=1$ yields a valid FDR-controlling algorithm, this choice can be conservative \citep{fischer2024onlineebh}. To select $\lambda>1$, we apply the robust-core CRC algorithm using a calibration dataset $\bcal^{\cs}\coloneqq \{\run^{\cs}_1,\ldots,\run^{\cs}_{n_{\cs}}\}$.
We assume that the runs in $\bcal^\cs$ and $\testrun$ are exchangeable and that their joint collection is independent of $\bcal^{\es}$. Using $\bcal^{\es}$, we apply the construction of Section~\ref{subsec:uncalibrated-rule} to each conformal calibration run $\run_i^{\cs}$ in lieu of $\testrun$. For $i\in\{1, \ldots, n_{\cs}\}$, this produces $\widehat E_t^{(i)}$, $R_t^{(i)}(\lambda)$, and $\tau_\kappa^{(i)}(\lambda)$. Unindexed quantities $\widehat E_t$, $R_t(\lambda)$, and $\tau_\kappa$ continue to refer to $\testrun$.

Fix a finite ordered grid $\Lambda=\{\lambda_1< \cdots<\lambda_M\}\subset(1,\infty)$ and $\Lambda_0\coloneqq\{0\}\cup\Lambda$, with the convention $R_t^{(i)}(0)=\emptyset$. For run $i$, define the loss function
\begin{equation}\label{eq:cal-loss}
  \ell_i(\lambda)
  \coloneqq
  \max_{1\le t\le\tau_\kappa^{(i)}(\lambda)}
  \mathrm{FDP}\bigl(R_t^{(i)}(\lambda)\bigr),
  \qquad
  \lambda\in\Lambda_0,
\end{equation}
so that $\mathrm{FDP}(R_\tau^{(i)}(\lambda))\le\ell_i(\lambda)$ whenever $\tau\le\tau_\kappa^{(i)}(\lambda)$. The fallback $\lambda=0$ makes no discoveries, so $\ell_i(0)=0$. Given that the runs in $\bcal^{\cs}$ are labeled, the loss profiles $\ell_1,\ldots,\ell_{n_{\cs}}$ are fully observable. The  loss profile $\ell$ for $\testrun$ is nonetheless unobserved as its oracle ground-truth labels are hidden. As detailed in Section~\ref{subsec:finite-grid}, applying the finite-grid robust-core procedure returns the most preferred $\widehat{\lambda}_{\mathrm{core}}\in\Lambda$, which remains valid for any test loss function $\ell$, or the fallback $0$ when the robust core is empty.

\begin{proposition}[Exchangeability of loss profiles]\label{prop:split-loss-exchangeable}
  Suppose the runs in $\bcal^{\cs}\cup\{\testrun\}$ are exchangeable and jointly independent of $\bcal^{\es}$. Applying the same construction to each run yields loss profiles $\ell_1,\ldots,\ell_{n_{\cs}},\ell$ that are exchangeable conditional on $\bcal^{\es}$ and hence marginally exchangeable.
\end{proposition}

\begin{theorem}[End-to-end trajectory-level FDR control]\label{thm:end-to-end}
  Fix $\kappa\ge1$ and assume the conditions of Proposition~\ref{prop:split-loss-exchangeable}. Let $\widehat{\lambda}_{\mathrm{core}}$ be the robust-core choice at level $\alpha$ computed from $\ell_1,\ldots,\ell_{n_{\cs}}$. Then
  \begin{equation}\label{eq:end-to-end-guarantee}
    \esp\bigl[
      \ell(\widehat{\lambda}_{\mathrm{core}})
    \bigr]
    \le \alpha,
  \end{equation}
  and consequently, for every unobserved-run stopping time $\tau\le\tau_{\kappa}(\widehat{\lambda}_{\mathrm{core}})$,
  \begin{equation*}
    \esp\bigl[
      \mathrm{FDP}\bigl(R_{\tau}(\widehat{\lambda}_{\mathrm{core}})\bigr)
    \bigr]
    \le \alpha.
  \end{equation*}
\end{theorem}

\subsection{End-to-end validation on synthetic data}\label{subsec:synthetic-end-to-end}

\begin{figure}[htbp]
  \centering
  \includegraphics[width=\linewidth]{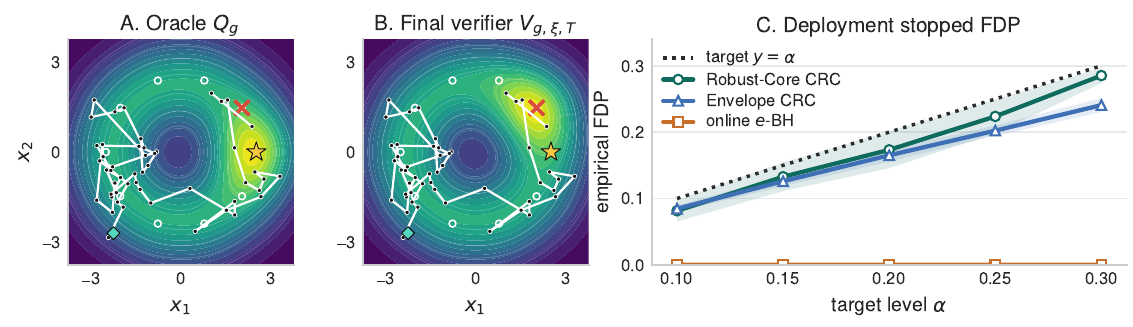}
  \caption{\textbf{Synthetic Gaussian-ring validation.} Panels A and B show the same sampled 50-round trajectory with no stopping rule applied, so they do not depend on $\alpha$. Panel A shows the oracle field $Q_g\coloneqq B+G_g$, which coincides with the initial verifier field $V_{g,\xi,1}$, and Panel B shows the final verifier field $V_{g,\xi,T}$. The filled contours give the value of the field at each point of $\rset^2$, with brighter colors for larger values. The yellow star marks $\mu_g$, and the red cross marks $\mu_{g+\xi}$. Panel C reports the mean stopped FDP against the target level $\alpha$, with 95\% Monte Carlo intervals over 20 independent repetitions.}
  \label{fig:sec6-endtoend}
\end{figure}

We now illustrate the full procedure on a synthetic example. In this setup, a generator moves in the plane toward regions where the verifier score is high. We define a proposal to be correct if it is closer to a hidden target center $\mu_g$ than to any other center. Over the run, the verifier moves its extra reward from the target center to a neighboring center, so high verifier scores become less and less indicative of correctness. Let $\mu_0,\ldots,\mu_9$ be equally spaced on the circle of radius $R=2.5$ centered at the origin in $\rset^2$. Each run draws a hidden target index $g\in\{0,\ldots,9\}$ and a neighbor offset $\xi\in\{-1,+1\}$, with indices taken modulo $10$. Writing $G_k(x)=\exp(-\|x-\mu_k\|^2/2)$ and $B(x)=\sum_{k=0}^{9}G_k(x)$, the time-$t$ verifier score is
\begin{equation*}
  V_{g,\xi,t}(x)=B(x)+w_tG_g(x)+(1-w_t)G_{g+\xi}(x),
  \qquad
  w_t=1-\frac{t-1}{T-1}.
\end{equation*}
In words, the verifier starts by placing additional mass on $\mu_g$, which is then transferred to $\mu_{g+\xi}$ over the run.

Each trajectory uses $T=50$ rounds and starts from a randomly drawn point $X_0\sim\mathcal{N}(0,R^2I_2)$. At round $t$, the generator draws candidate locations $Z_{t,1},\ldots,Z_{t,12}$ conditionally independently according to
\begin{equation*}
  Z_{t,j}\mid X_{t-1}
  \overset{\mathrm{ind}}{\sim}
  \mathcal{N}\bigl(X_{t-1},\sigma_{\mathrm{prop}}^2I_2\bigr),
  \qquad
  j\in\{1,\ldots,12\},
\end{equation*}
where $\sigma_{\mathrm{prop}}=R\sin(\pi/10)$, which is proportional to the distance between two neighboring points on the circle.  Conditional on the candidates, the generator samples an index $J_t$ with
\begin{equation*}
  \mathbb{P}\bigl(J_t=j\mid Z_{t,1},\ldots,Z_{t,12}\bigr)
  =
  \frac{V_{g,\xi,t}(Z_{t,j})}
       {\sum_{\ell=1}^{12}V_{g,\xi,t}(Z_{t,\ell})},
  \qquad
  j\in\{1,\ldots,12\},
\end{equation*}
and sets $X_t=Z_{t,J_t}$. The observed proxy score is $S_t=V_{g,\xi,t}(X_t)$. The oracle label is $Y_t=1$ if $X_t$ is closer to $\mu_g$ than to any other basin center, and $Y_t=0$ otherwise. Thus, a trajectory can continue receiving high verifier scores while moving toward the incorrect neighboring basin $\mu_{g+\xi}$. This design aims to simulate a setting where the verifier and ground-truth correctness are increasingly misaligned as rounds pass.

We repeated the experiment for $20$ runs. For each run, we use $100$ sampled trajectories for the $e$-value construction split $\bcal^{\es}$, $500$ trajectories for the conformal calibration split $\bcal^{\cs}$, and $350$  trajectories as evaluation runs. Each evaluation run plays the role of $\testrun$ with its labels withheld while the stopping rule is applied. Its correctness labels are then used only to compute stopped FDP.

The round-wise $e$-values use the feature $\log S_t$ and a roundwise likelihood-ratio estimate based on one-dimensional Gaussian kernel density estimation. Online $e$-BH uses the uniform schedule $\gamma_t=1/T$ with $\kappa=1$. We compare Robust-Core CRC, Envelope CRC, and the unscaled online $e$-BH. Robust-Core CRC selects over the uniform grid $\{0,1,\ldots,16\}$. Envelope CRC is optimized over the continuous interval $[0,16]$ using the changepoints of the online $e$-BH loss. The reported loss is the FDP at stopping, averaged over the $350$ evaluation runs. Bands in Figure~\ref{fig:sec6-endtoend} are the mean $\pm 2$ standard errors over the $20$ repetitions of the experiment at $\alpha\in\{0.10,0.15,0.20,0.25,0.30\}$.

Panel C of Figure~\ref{fig:sec6-endtoend} plots FDP at stopping against the target level $\alpha$. The diagonal $y=\alpha$ marks a procedure that spends its error budget exactly. Online $e$-BH sits far below the diagonal because it fixes $\lambda=1$. At $\alpha=0.10$, Robust-Core CRC and Envelope CRC have similar mean FDPs of $0.082$ and $0.085$. Robust-Core CRC is closer to the diagonal for larger $\alpha$ levels but the difference between Envelope CRC and Robust-Core CRC is very minor at smaller $\alpha$ values. 

\section{Experiments}\label{sec:experiments}

\begin{figure}[htbp]
  \centering
  \includegraphics[width=\linewidth]{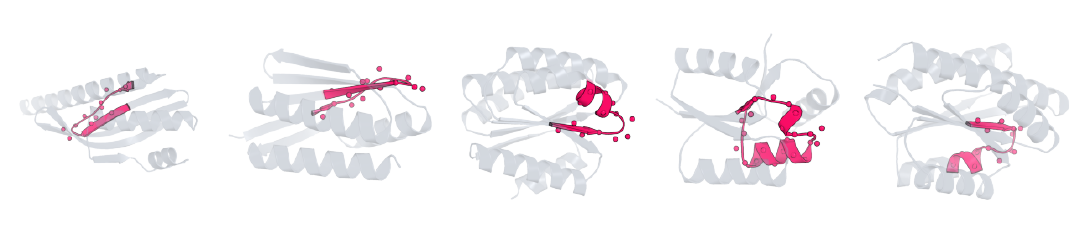}
  \caption{\textbf{Motif scaffolding examples.} Each column shows a MotifBench-derived
    task after scaffolding. The fixed motif is magenta and the designed scaffold is
  grey.}
  \label{fig:protein-motif}
\end{figure}

To illustrate the utility of our approach in a more realistic scenario, we consider a sequential protein-design setting, where the generator adapts to a cheap verifier. Informally, a protein generative model proposes a set of particles in $\rset^3$ representing the positions of the protein's backbone atoms. In motif scaffolding, some of these particles form a fixed functional motif, which the model aims to surround with a supporting section of the protein, also known as the scaffold. Our aim is to determine whether the proposed three-dimensional structure is designable. A three-dimensional structure is designable if there exists an amino-acid sequence that plausibly folds into it.

A common computational test of designability compares the proposed protein backbone with the result of a sequence-design-and-refolding procedure~\citep{zheng2025motifbench}. First, a generative model proposes a three-dimensional structure of the protein backbone.  Then, an inverse-folding model predicts amino-acid sequences that could produce the protein backbone proposed by the generative model. A structure-prediction model, such as AlphaFold3 \citep{abramson2024accurate}, then folds each predicted sequence independently, producing another set of structures in $\rset^3$. If the refolded structure closely matches the original proposal, this provides evidence that the proposed protein backbone is physically plausible in the sense that it can be realized by an amino-acid sequence.

\begin{figure}[htbp]
  \centering
  \includegraphics[width=\linewidth]{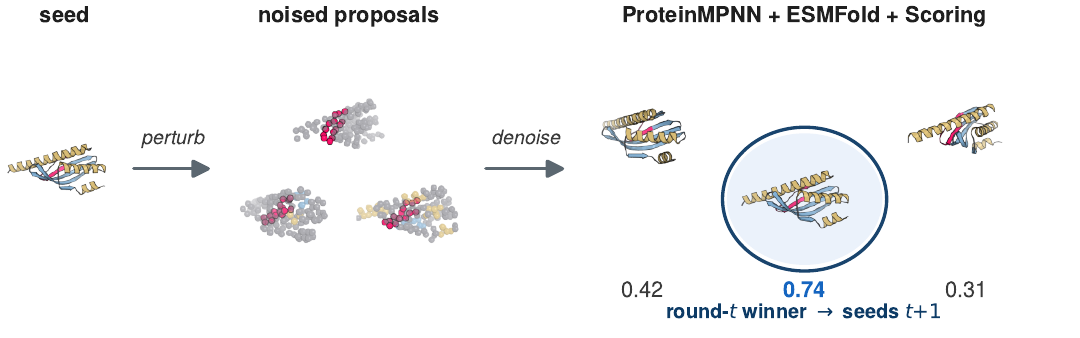}
  \caption{\textbf{One round of the protein feedback loop.} Four exploit candidates continue from the stored midpoint state associated with the previous winner, while four explore candidates start from the original motif specification. The cheap ProteinMPNN--ESMFold pipeline scores the resulting proposals, and the midpoint state associated with the highest-scoring proposal seeds round $t{+}1$.}
  \label{fig:protein-loop}
\end{figure}

We measure structural agreement using root-mean-square deviation (RMSD). For each comparison, we rigidly align the relevant regions of the proposed and refolded structures. RMSD is the square root of the average squared distance between corresponding atoms after alignment, so smaller values indicate closer structural agreement. The score uses three RMSD terms. $\mathrm{RMSD}_{\mathrm{bb\text{-}ESM}}$ compares the complete proposed protein backbone with the refolded one. $\mathrm{RMSD}_{\mathrm{scaffold}}$ measures agreement across the generator's recorded scaffold region. $\mathrm{RMSD}_{\mathrm{motif}}$ measures how well the refolded structure preserves the prescribed motif geometry. Appendix~\ref{app:protein-feedback} specifies the exact atoms and structural masks used for each term. The score also includes $\mathrm{pLDDT}_{\mathrm{mean}}$, the mean predicted local distance difference test score, and $\mathrm{pTM}$, the predicted Template Modeling score. Larger $\mathrm{pLDDT}_{\mathrm{mean}}$ and  $\mathrm{pTM}$ values indicate greater confidence in the local geometry and overall fold, respectively.

We compute these measurements using a cheap verification pipeline. The Protein Message Passing Neural Network (ProteinMPNN)~\citep{dauparas2022robust} predicts eight amino-acid sequences for each proposed protein backbone. ESMFold~\citep{lin2023evolutionary}, a computationally light structure-prediction model, folds each sequence independently and therefore produces eight refolded structures. For each sequence--refold pair, the cheap verifier computes the three RMSDs, mean pLDDT, and pTM, then combines them into an individual proxy score
\begin{equation*}
  S
  =
  \mathrm{pLDDT}_{\mathrm{mean}}
  + 100\,\mathrm{pTM}
  - 10\,\mathrm{RMSD}_{\mathrm{bb\text{-}ESM}}
  - 5\,\mathrm{RMSD}_{\mathrm{scaffold}}
  - 20\,\mathrm{RMSD}_{\mathrm{motif}} .
\end{equation*}

We note that this score function was not particularly tuned to correlate with a specific notion of ground-truth correctness. Rather, the weighting of the different terms was heuristically selected to balance the scale and importance of different terms. The first two terms reward refolds that ESMFold considers reliable, while the RMSD terms penalize disagreement with the intended structure. Motif disagreement receives the largest penalty because preserving the prescribed motif geometry is the central constraint of the task.

\begin{figure}[htbp]
  \centering
  \includegraphics[width=\linewidth]{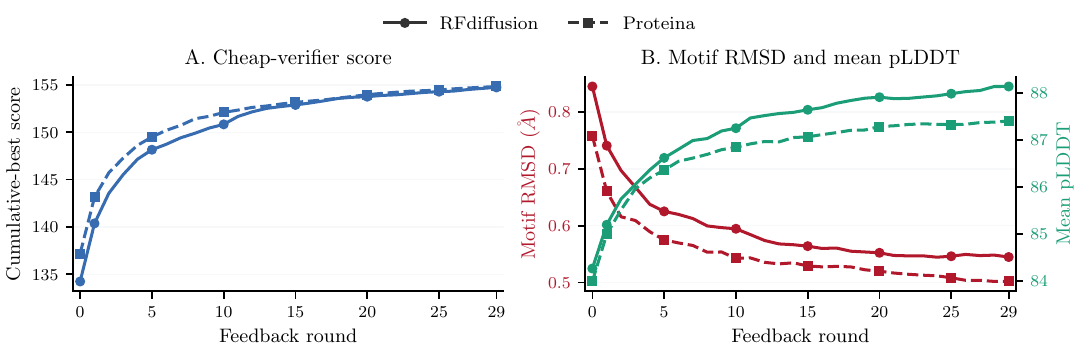}
  \caption{\textbf{Cheap-verifier trends across rounds.}
    At each round, we retain the design with the highest cheap-verifier score seen
    in each trajectory. Solid circles denote RFdiffusion and dashed squares denote Proteina. Panel A reports cumulative-best score. Panel B reports ESMFold motif RMSD on the left axis and mean pLDDT on the right axis.}
  \label{fig:protein-verifier-trends}
\end{figure}

We use this cheap verification score inside a 30-round feedback loop with RoseTTAFold diffusion (RFdiffusion)~\citep{watson2023novo} and Proteina~\citep{geffner2025proteina}. Both models generate structures by progressively transforming a noisy intermediate state into three-dimensional coordinates. During each trajectory, we save the generator's partially noised state at the midpoint. The first round contains eight fresh candidates from the original motif specification. Every later round contains four fresh explore candidates and four exploit candidates. The explore candidates are generated from scratch while the exploit candidates are generated from the midpoint state corresponding to the previous best-scoring structure. Appendix~\ref{app:protein-feedback} gives the detailed description of the feedback structure.

Figure~\ref{fig:protein-verifier-trends} shows rapid gains during the early feedback rounds followed by a gradual plateau for both generators. The score curves are nondecreasing by construction because each trajectory retains its best scoring candidate by design. Motif RMSD generally decreases over the same rounds, while mean pLDDT increases. Proteina reaches a lower motif RMSD, whereas RFdiffusion attains a higher mean pLDDT. These trends show that feedback finds proposals with better motif preservation and higher confidence, as scored through the cheap verification. Thus, as rounds pass, there is a risk of favoring proposals well-scored by the cheap verifier rather than properties associated with ground-truth correctness.

\begin{figure}[htbp]
  \centering
  \includegraphics[width=\linewidth]{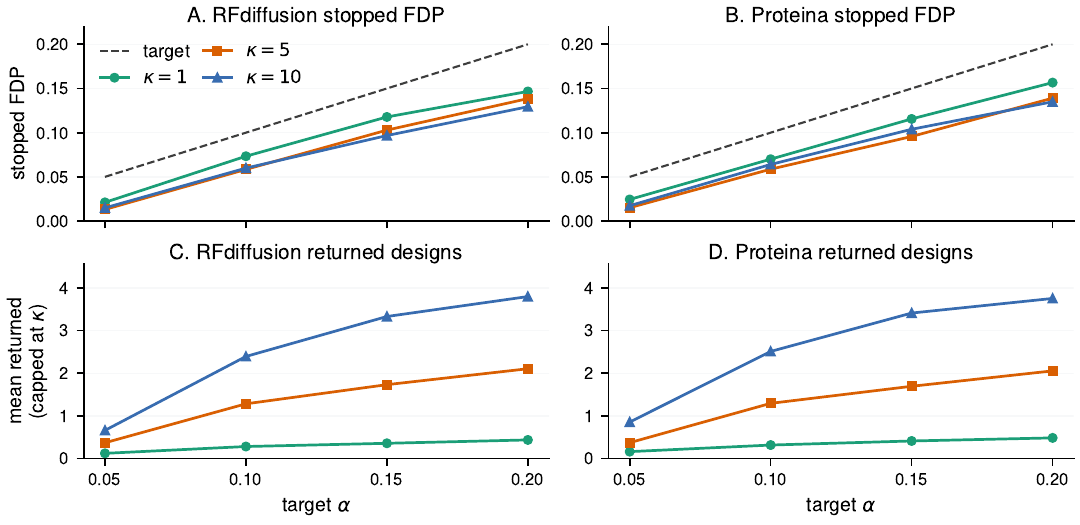}
  \caption{\textbf{Stopped risk and utility on held-out trajectories.}
    The robust-core calibrated online $e$-BH rule is evaluated over target levels
    $\alpha\in\{0.05,0.10,0.15,0.20\}$ and target acceptance counts
    $\kappa\in\{1,5,10\}$. Panels A and B report stopped FDP for RFdiffusion
    and Proteina. Panels C and D report returned designs, with counts capped at
  $\kappa$, for the same two generators.}
  \label{fig:protein-results}
\end{figure}

For ground-truth correctness, we use AlphaFold~3~\citep{abramson2024accurate} as a more computationally expensive oracle. A selected proposal is successful when the AlphaFold~3 prediction for its associated amino-acid sequence, predicted by ProteinMPNN, matches the prescribed motif within 0.75~\AA{} RMSD, matches the proposed backbone within 1.5~\AA{} C$\alpha$ RMSD over the scaffold mask, has pTM at least $0.80$, and has mean pLDDT at least $85$. We also note that these criteria were heuristically selected and the pipeline may be reimplemented with different metrics for ground truth correctness.

We generated $300$ labeled trajectories for each of Proteina and RFdiffusion. A single trajectory is a $30$-round feedback run for a fixed motif-scaffolding task. For each held-out trajectory, the remaining $299$ trajectories from the same generator are split into $199$ for $e$-value construction and $100$ for conformal calibration. The stopping rule is replayed on the held-out trajectory with its AlphaFold~3 labels withheld. Stopped FDP and the number of returned designs are scored afterward. Additional details are provided in Appendix~\ref{app:experiments}.

Figure~\ref{fig:protein-results} shows the stopped FDP of the full procedure. In Panels A and B, stopped FDP increases as $\alpha$ increases, while the curves for different values of $\kappa$ remain close. Thus, allowing more accepted designs does not substantially increase the observed error rate. Panels C and D show the corresponding increase in utility, with larger values of both $\alpha$ and $\kappa$ returning more designs.

\section{Conclusions}\label{sec:discussion}

In this paper, we study how to safely stop generator--verifier loops that use a cheap verifier to scale inference-time computation. While repeated verifier feedback allows the generator to efficiently explore many proposals without repeatedly querying an expensive oracle, it can also bias the search towards samples that are well-scored by the verifier but not by the ground-truth oracle. In our proposed approach, we cast this setting as an online multiple-testing problem and combine index-betting $e$-values, robust-core conformal risk control, and online $e$-BH to solve it. The resulting procedure controls stopped FDR up to the $\kappa$-acceptance time while permitting arbitrary dependence within a run. Our synthetic experiments validate the statistical construction, and the protein-design study shows how it can be applied to a more realistic setting.

We foresee two possible extensions of our approach. First, the current guarantee requires the logged and deployed runs to be exchangeable or even i.i.d. This requirement remains a limitation, as in many real-world examples data arrives sequentially and may not be exchangeable. A weighted form of conformal risk control~\citep{prinster2026conformal}, inspired by weighted conformal prediction under covariate shift~\citep{tibshirani2019conformal} and more general distribution shifts~\citep{barber2023conformal,prinster2024conformal}, could extend our results to the non-exchangeable setting. Nonetheless, combining weighted conformal prediction with our non-monotone construction remains an open problem. Second, many inference-time systems organize proposals as a tree or directed acyclic graph rather than a single trajectory. Extending our framework from looped trajectories to more adaptive tree or graph construction would require FDR control under adaptive branching, pruning, and merging of different proposed solutions.  

\section*{Acknowledgements}

We thank Drew Prinster for a careful reading of an earlier version and for detailed feedback and discussion, which improved the presentation and content of this manuscript. The work of Aymeric Dieuleveut and Mahmoud Hegazy is supported by French State aid managed by the Agence Nationale de la Recherche (ANR) under the France 2030 program with the reference ANR-23-PEIA-005 (REDEEM project), and ANR-23-IACL-0005, in particular the Hi!Paris FLAG chair. This project is part of the Chair ``Markets and Learning,'' supported by Air Liquide, BNP PARIBAS ASSET MANAGEMENT Europe, EDF, Orange and SNCF, sponsors of the Inria Foundation. We also acknowledge access to the HPC resources of IDRIS under the allocations AD011017341 and A0201017541 by GENCI. Finally, the work of Michael Jordan is supported by the European Union (ERC-2022-SYG-OCEAN-101071601).
Views and opinions expressed are however those of the authors only and do not necessarily reflect those of the European Union or the European Research Council Executive Agency. Neither the European Union nor the granting authority can be held responsible for them. 

\bibliographystyle{plainnat} \bibliography{ref}

\clearpage \appendix
\section{Deferred Proofs}\label{app:deferred-proofs}\label{app:sec4-proofs}

Throughout the proof subsections for Section~\ref{sec:marginal}, $n$ denotes the number of logged null windows. Accordingly, the augmented null batch has size $n+1$, which corresponds to $n_{0,t}$ in the main text. We fix $t$ and write $P\coloneqq P_{0,t}$, $Q\coloneqq P_{1,t}$, and $L\coloneqq \mathrm dQ/\mathrm dP$. 

\subsection{Proposition~\ref{prop:kelly-lr}}
\begin{proof}
  All probabilities and expectations below are under the alternative. By the i.i.d.~assumption of Section~\ref{sec:overview}, $Z_0\sim Q$ is independent of $Z_1,\ldots,Z_n\overset{\mathrm{iid}}{\sim}P$. Let $\Pi$ be a uniformly random permutation of $\{0,\ldots,n\}$, independent of $(Z_0,\ldots,Z_n)$. Write
  \begin{equation*}
    \widetilde Z_k\coloneqq Z_{\Pi(k)},
    \qquad
    J^\star\coloneqq\Pi^{-1}(0).
  \end{equation*}
  Thus $\widetilde Z=(\widetilde Z_0,\ldots,\widetilde Z_n)$ is the ordered batch and $J^\star$ is the position occupied by $Z_0$. Since $J^\star$ is a function of $\Pi$ alone, it is uniform on $\{0,\ldots,n\}$ and independent of $(Z_0,\ldots,Z_n)$.

Hence, conditional on $J^\star=j$, the different coordinates of $\widetilde Z$ are independent, with $\widetilde Z_j\sim Q$ and $\widetilde Z_\ell\sim P$ for every $\ell\neq j$. Bayes' rule gives
  \begin{equation*}
    \Pr(J^\star=j\mid\widetilde Z)
    =
    \frac{
      \Pr(\widetilde Z\mid J^\star=j)\Pr(J^\star=j)
    }{
      \sum_{\ell=0}^n
      \Pr(\widetilde Z\mid J^\star=\ell)\Pr(J^\star=\ell)
    }
    =
    \frac{L(\widetilde Z_j)/(n+1)}
    {\sum_{\ell=0}^n L(\widetilde Z_\ell)/(n+1)}
    =
    \frac{L(\widetilde Z_j)}
    {\sum_{\ell=0}^n L(\widetilde Z_\ell)},
    \qquad j=0,\ldots,n,
  \end{equation*}
  almost surely.

  Consider $q\in\rset^{n+1}$ such that $q_j = \Pr(J^\star=j\mid\widetilde Z)$. For any $(Z_0,\ldots, Z_n)$, any maximizing budgeted strategy $\beta$ must satisfy $\beta_j>0$ whenever $q_j>0$, since otherwise its conditional expected log-wealth is $-\infty$ and satisfy $\sum_{j=0}^n\beta_j=1$ (otherwise rescaling increases the conditional expected log-wealth). Thus it remains to optimize over the simplex, where
  \begin{equation*}
    \esp\!\left[\log\!\big((n+1)\beta_{J^\star}\big)\mid \widetilde Z\right]
    =
    \log(n+1)
    +
    \sum_{j=0}^n q_j\log q_j
    -
    D_{\mathrm{KL}}(q\|\beta).
  \end{equation*}
  Nonnegativity of the KL divergence shows that the unique maximizer is $\beta=q$. Substituting the posterior above yields
  \begin{equation*}
    \beta_j^{\mathrm{LR}}(\widetilde Z)
    =
    \frac{L(\widetilde Z_j)}
    {\sum_{\ell=0}^n L(\widetilde Z_\ell)},
    \qquad j=0,\ldots,n.
  \end{equation*}
  This rule is batch-symmetric, so it also maximizes conditional expected log-wealth within the batch-symmetric class.
\end{proof}

\subsection{Theorem~\ref{thm:power-decomp}}
\begin{proof}
  Let $Z\sim Q$ and let $Z_1^0,\ldots,Z_n^0\overset{\mathrm{iid}}{\sim}P$ be independent of $Z$. Write $L_0\coloneqq L(Z)$, $\widehat L_0\coloneqq \widehat L_{-0}(Z)$, $L_j\coloneqq L(Z_j^0)$, and $\widehat L_j\coloneqq \widehat L_{-j}(Z_j^0)$ for $1\le j\le n$. Then
  \begin{equation*}
    \log E^{\mathrm{idx}}
    =
    \log(n+1)+\log L_0-\log\!\Big(\sum_{j=0}^n L_j\Big),
  \end{equation*}
  and similarly
  \begin{equation*}
    \log \widehat E^{\mathrm{idx}}
    =
    \log(n+1)+\log \widehat L_0-\log\!\Big(\sum_{j=0}^n \widehat L_j\Big).
  \end{equation*}
  Hence
  \begin{equation*}
    \left|\log E^{\mathrm{idx}}-\log \widehat E^{\mathrm{idx}}\right|
    \le
    \left|\log L_0-\log \widehat L_0\right|
    +
    \left|
    \log\!\Big(\sum_{j=0}^n L_j\Big)-\log\!\Big(\sum_{j=0}^n \widehat L_j\Big)
    \right|.
  \end{equation*}
  Set
  \begin{equation*}
    m\coloneqq \max_{0\le j\le n}\left|\log L_j-\log \widehat L_j\right|.
  \end{equation*}
  Then $e^{-m}\widehat L_j\le L_j\le e^m \widehat L_j$ for every $j$, so after summing,
  \begin{equation*}
    e^{-m}\sum_{j=0}^n \widehat L_j
    \le
    \sum_{j=0}^n L_j
    \le
    e^m\sum_{j=0}^n \widehat L_j.
  \end{equation*}
  Therefore
  \begin{equation*}
    \left|
    \log\!\Big(\sum_{j=0}^n L_j\Big)-\log\!\Big(\sum_{j=0}^n \widehat L_j\Big)
    \right|
    \le m,
  \end{equation*}
  and the claim follows.
\end{proof}

\subsection{Theorem~\ref{thm:trunc-power-gap}}
\begin{proof}
  Let $Z\sim Q$, let $Z_1^0,\ldots,Z_n^0\overset{\mathrm{iid}}{\sim}P$, independent of $Z$, and set $L_0\coloneqq L(Z)$, $L_j\coloneqq L(Z_j^0)$, and $S_n\coloneqq \sum_{j=1}^n L_j$. Since
  \begin{equation*}
    E^{\mathrm{idx}}=(n+1)\frac{L_0}{L_0+S_n},
  \end{equation*}
  we have
  \begin{equation*}
    \Delta_n
    \coloneqq
    \log L_0-\log E^{\mathrm{idx}}
    =
    \log\!\Big(\frac{L_0+S_n}{n+1}\Big).
  \end{equation*}

  Condition on $L_0$. By Jensen's inequality,
  \begin{equation*}
    \mathbb E[\Delta_n\mid L_0]
    \le
    \log\!\Big(\frac{L_0+\mathbb E_P[S_n]}{n+1}\Big)
    =
    \log\!\Big(1+\frac{L_0-1}{n+1}\Big),
  \end{equation*}
  because $\mathbb E_P[L]=1$, hence $\mathbb E_P[S_n]=n$.

  Since $u\mapsto\log(1+u)$ is increasing and $(L_0-1)/(n+1)\ge -1$,
  \begin{equation*}
    \log\!\Big(1+\frac{L_0-1}{n+1}\Big)
    \le
    \log\!\Big(1+\frac{(L_0-1)_+}{n+1}\Big).
  \end{equation*}
  Set $C_\delta\coloneqq\sup_{x>0}\log(1+x)/x^\delta$. By the definition of $C_\delta$,
  \begin{equation*}
    \log(1+x)\le C_\delta x^\delta,\qquad x\ge0.
  \end{equation*}
  Applying this with $x=(L_0-1)_+/(n+1)$ yields
  \begin{equation*}
    \mathbb E[\Delta_n\mid L_0]
    \le
    \frac{C_\delta (L_0-1)_+^\delta}{(n+1)^\delta}.
  \end{equation*}
  Taking expectation under $Q$ and using $\mathbb E_{1}[g(L_0)]=\mathbb E_{P}[L\,g(L)]$,
  \begin{equation*}
    \mathbb E_{1}[\Delta_n]
    \le
    \frac{C_\delta\,\mathbb E_{1}[(L_0-1)_+^\delta]}{(n+1)^\delta}
    =
    \frac{C_\delta\,\mathbb E_{P}[L(L-1)_+^\delta]}{(n+1)^\delta}
    \le
    \frac{C_\delta\,\mathbb E_{P}[L^{1+\delta}]}{(n+1)^\delta}.
  \end{equation*}
\end{proof}

\subsection{Theorem~\ref{thm:lower-bound}}
\begin{proof}
  Fix $a=\frac{B-1}{2}$. For each sufficiently large $n$, let $L$ have the two-point law
  \begin{equation*}
    P(L=n)=p\coloneqq a n^{-(1+\delta)},
    \qquad
    P(L=y)=1-p,
    \qquad
    y\coloneqq\frac{1-np}{1-p}.
  \end{equation*}
  Then $0<y<1$ and $\esp_P[L]=1$, so $\mathrm dQ=L\,\mathrm dP$ defines a probability distribution. Moreover,
  \begin{equation*}
    \esp_P[L^{1+\delta}]
    =
    a+y^{1+\delta}(1-p)
    \longrightarrow
    a+1
    =
    \frac{B+1}{2}
    <B.
  \end{equation*}
  Thus $(P,Q)\in\mathcal{P}_{1+\delta}(B)$ for all sufficiently large $n$. Under $Q$,
  \begin{equation*}
    Q(L=n)=np=a n^{-\delta},
    \qquad
    Q(L=y)=1-np.
  \end{equation*}

  Let $L_0\sim Q$ and $L_1,\ldots,L_n\overset{\mathrm{iid}}{\sim}P$ be independent, and let $H$ count how many of $L_0,\ldots,L_n$ equal $n$. The average $(L_0+\cdots+L_n)/(n+1)$ equals $y$ when $H=0$ and $n(1+y)/(n+1)$ when $H=1$. It exceeds one when $H\ge2$ and $n$ is large. Therefore,
  \begin{equation*}
    \begin{split}
      \esp_1\!\left[\log L_0-\log E^{\mathrm{idx}}\right]
      &=
      \esp_1\!\left[\log\frac{L_0+\cdots+L_n}{n+1}\right]\\
      &\ge
      (1-np)(1-p)^n\log y\\
      &\quad+
      \left[np(1-p)^n+(1-np)np(1-p)^{n-1}\right]
      \log\!\left(\frac{n(1+y)}{n+1}\right).
    \end{split}
  \end{equation*}
  Since $\delta\in(0,1)$,
  \begin{equation*}
    \log y=-a n^{-\delta}+O(n^{-2\delta}),
    \qquad
    (1-p)^n=1-a n^{-\delta}+O(n^{-2\delta}),
    \qquad
    \log\!\left(\frac{n(1+y)}{n+1}\right)
    =\log 2+O(n^{-\delta}).
  \end{equation*}
  Substitution gives
  \begin{equation*}
    \esp_1\!\left[\log L_0-\log E^{\mathrm{idx}}\right]
    \ge
    a(2\log 2-1)n^{-\delta}+O(n^{-2\delta}).
  \end{equation*}
  Since $2\log 2-1>0$, the right-hand side is at least $\frac{a}{2}(2\log 2-1)n^{-\delta}$ for all sufficiently large $n$. Finally, Proposition~\ref{prop:kelly-lr} shows that $E^{\mathrm{idx}}$ minimizes this gap among budgeted batch-symmetric strategies, so the same lower bound applies to every such strategy.
\end{proof}

\subsection{Corollary~\ref{cor:kl-dominance}}
\begin{proof}
  Let $Z_0,\ldots,Z_n\overset{\mathrm{iid}}{\sim}P$, set $L_j\coloneqq L(Z_j)$ for $0\le j\le n$, and define
  \begin{equation*}
    A_n\coloneqq \frac{1}{n+1}\sum_{j=0}^n L_j.
  \end{equation*}
  By change of measure and exchangeability,
  \begin{align*}
    \mathbb E_{1}\!\left[\log L_0-\log E^{\mathrm{idx}}\right]
    &=
    \mathbb E_{P}^{\otimes(n+1)}\!\left[L_0\log A_n\right] \\
    &=
    \frac{1}{n+1}\sum_{j=0}^n
    \mathbb E_{P}^{\otimes(n+1)}\!\left[L_j\log A_n\right] \\
    &=
    \mathbb E_{P}^{\otimes(n+1)}\!\left[A_n\log A_n\right].
  \end{align*}

  Since $\mathbb E_{P}[L]=1$, the strong law of large numbers gives $A_n\to1$ almost surely, hence $A_n\log A_n\to0$ almost surely.

  Now let $\phi(x)\coloneqq x\log x$, with $0\log0=0$. The function $\phi$ is convex and bounded below by $-e^{-1}$. By Jensen,
  \begin{equation*}
    \phi(A_n)
    \le
    \frac{1}{n+1}\sum_{j=0}^n \phi(L_j),
  \end{equation*}
  and hence
  \begin{equation*}
    \phi_+(A_n)
    \le
    \frac{1}{n+1}\sum_{j=0}^n \phi_+(L_j).
  \end{equation*}
  Since $\mathbb E_P[\phi_+(L)]<\infty$, the empirical averages $(n+1)^{-1}\sum_{j=0}^n \phi_+(L_j)$ are uniformly integrable. Therefore $\{\phi_+(A_n)\}_{n\ge0}$ is uniformly integrable and $\{\phi(A_n)\}_{n\ge0}$ is uniformly integrable. Since $\phi(A_n)\to0$ almost surely, it follows that
  \begin{equation*}
    \mathbb E[A_n\log A_n]\to 0.
  \end{equation*}
\end{proof}

\subsection{Theorem~\ref{thm:tree-crc-valid}}\label{app:sec5-proofs}
\begin{proof}
Let $C\coloneqq G(\ell_1,\ldots,\ell_{n+1})$. Since the true loss profile $\ell_{n+1}$ belongs to $\mathcal{L}$ almost surely, it is one of the completions in the intersection defining the robust core. Hence $C_n\subseteq C$.

If $C_n\neq\emptyset$, the selected parameter belongs to $C$, so its test loss is bounded by $L_{n+1}(C)$. If $C_n=\emptyset$, the fallback incurs zero loss by Equation~(\ref{eq:safe-choice}). Nonnegativity of the set loss, with $L_{n+1}(\emptyset)=0$, therefore gives
\begin{equation*}
  \ell_{n+1}(\widehat{\lambda}_{\mathrm{core}})
  \le L_{n+1}(C)
\end{equation*}
almost surely.

Symmetry of $G$ and exchangeability of the loss profiles imply that $L_1(C),\ldots,L_{n+1}(C)$ are identically distributed. Taking expectations and using Equation~(\ref{eq:abstract-feasibility}),
\begin{equation*}
  \begin{aligned}
    \esp[\ell_{n+1}(\widehat{\lambda}_{\mathrm{core}})]
    &\le \esp[L_{n+1}(C)] \\
    &= \esp\left[\frac{1}{n+1}\sum_{i=1}^{n+1}L_i(C)\right]
    \le \alpha.
  \end{aligned}
\end{equation*}
\end{proof}

\subsection{Lemma~\ref{lem:crc-tree-correct}}
\begin{proof}
Consider the full scan without early termination. Suppose the frontier before processing $\lambda_k$ represents all attainable pairs of calibration maxima $M$ and held-out maxima $u$, with each state $(M,[a,b])$ representing $\{M\}\times[a,b]$. Initially, the only attainable pair is $(0^n,0)$. The unprocessed coordinates of $w$ remain unrestricted because $\mathcal{U}=[0,1]^\Lambda$.

For a state $(M,[a,b])$, let $M'$ and $\Theta$ be defined as in Algorithm~\ref{alg:crc-tree}. Inclusion requires $\max\{u,w(\lambda_k)\}\le\Theta$. This is possible if and only if $a\le\Theta$, and the new held-out maximum ranges over $[a,\min\{1,\Theta\}]$. Every value $r$ in this interval is attained by taking $u=a$ and $w(\lambda_k)=r$, giving the successor $(M',[a,\min\{1,\Theta\}])$.

Skipping is possible if and only if $\Theta<1$. Choosing $w(\lambda_k)=1$ then produces a skip for every $u\in[a,b]$, leaving the state $(M,[a,b])$ unchanged. Merging identical successor states changes neither the attainable pairs nor the freedom to choose the remaining coordinates.

Thus $\mathrm{AllInclude}$ is true at $\lambda_k$ if and only if every completion includes that point. Later decisions cannot remove it, so this condition is equivalent to membership in the robust core. Since the scan proceeds from most to least preferred, it returns the largest member of the core, or the fallback when no point belongs to it.
\end{proof}

\subsection{Proposition~\ref{prop:split-loss-exchangeable} and Theorem~\ref{thm:end-to-end}}\label{app:end-to-end-proof}
\begin{proof}
  Work conditionally on the $e$-value construction split $\bcal^{\es}$ and fix $\kappa\ge 1$. Independence of the conformal calibration runs and the unobserved run from $\bcal^{\es}$ preserves their exchangeability under this conditioning. With $\bcal^{\es}$ fixed, the construction in Sections~\ref{subsec:uncalibrated-rule} and~\ref{subsec:robust-core-calibration} applies the same map to every such run. Any permutation of the $n_{\cs}$ conformal calibration runs and the unobserved run therefore permutes the local $e$-value streams in the same way. For each $\lambda\in\Lambda_0$, the online $e$-BH path, the stopping time, and the process loss are deterministic functions of the corresponding $e$-value stream and run labels. The same permutation consequently relabels $\ell_1,\ldots,\ell_{n_{\cs}},\ell$. Thus, they are conditional exchangeable. Integrating over $\bcal^{\es}$ gives marginal exchangeability and completes the proof of Proposition~\ref{prop:split-loss-exchangeable}.

  By Lemma~\ref{lem:crc-tree-correct}, Algorithm~\ref{alg:crc-tree} returns the robust-core estimator from Equation~(\ref{eq:robust-core-estimator}) when applied to the conformal calibration profiles $\ell_1,\ldots,\ell_{n_{\cs}}$. Since the fallback satisfies $\ell_i(0)=0$ for $i=1,\ldots,n_{\cs}$ and $\ell(0)=0$, Theorem~\ref{thm:tree-crc-valid} applies conditionally on $\bcal^{\es}$ and yields
  \begin{equation*}
    \esp\left[
      \ell(\widehat{\lambda}_{\mathrm{core}})
      \middle| \bcal^{\es}
    \right]
    \le \alpha.
  \end{equation*}
  Taking expectations over $\bcal^{\es}$ removes the conditioning and gives Equation~(\ref{eq:end-to-end-guarantee}). For any unobserved-run stopping time $\tau\le\tau_{\kappa}(\widehat{\lambda}_{\mathrm{core}})$,
  \begin{equation*}
    \mathrm{FDP}
    (R_{\tau}(\widehat{\lambda}_{\mathrm{core}}))
    \le
    \ell(\widehat{\lambda}_{\mathrm{core}}).
  \end{equation*}
\end{proof}

\section{Additional Experiment Details}\label{app:experiments}

For each of the 300 motif-scaffolding targets, each generator produces a 30-round trajectory, and AlphaFold~3 evaluation takes place after generation.

\subsection{MotifBench-derived benchmark construction} \label{app:biolip}

We restrict ourselves to tasks whose motif is a single contiguous segment and use the ten tasks of Table~\ref{tab:protein-targets}, taken from MotifBench \citep{zheng2025motifbench}. Each MotifBench task specifies the fixed motif and integer ranges for the scaffold lengths on its left and right sides. For each task, we sample 30 distinct integer pairs of left and right scaffold lengths within the specified ranges, which gives $10\times 30=300$ targets. We use the same 300 targets for RFdiffusion and Proteina.

\begin{table}[H]
  \centering
  \small
  \begin{tabular}{llccc}
    \toprule
    Task & Motif residues & Motif length & Left range & Right range \\
    \midrule
    01\_1LDB & A186--206 & 21 & 0--100 & 0--100 \\
    02\_1ITU & A124--147 & 24 & 0--100 & 0--100 \\
    03\_2CGA & A184--194 & 11 & 0--100 & 0--100 \\
    04\_5WN9 & A170--189 & 20 & 10--40 & 10--40 \\
    05\_5ZE9 & A229--243 & 15 & 0--100 & 0--100 \\
    06\_6E6R & A23--35 & 13 & 0--65 & 0--65 \\
    07\_6E6R & A23--35 & 13 & 0--190 & 0--190 \\
    08\_7AD5 & A99--113 & 15 & 0--100 & 0--100 \\
    09\_7CG5 & A6--20 & 15 & 0--100 & 0--100 \\
    10\_7WRK & A80--94 & 15 & 0--100 & 0--100 \\
    \bottomrule
  \end{tabular}
  \caption{MotifBench source tasks used to construct the 300 targets, each fixing one pair of left and right scaffold lengths.}
  \label{tab:protein-targets}
\end{table}

\subsection{Protein feedback loop} \label{app:protein-feedback}

A feedback trajectory contains 30 rounds with eight candidate backbones per round. The first round generates all eight backbones from the motif specification, while each later round combines four exploit backbones sampled from the preceding winner's midpoint state with four explore backbones sampled from fresh trajectories.

ProteinMPNN~\citep{dauparas2022robust} samples eight sequences per backbone at temperature $0.1$ with no backbone noise. It uses the full N, C$\alpha$, C, and O backbone for RFdiffusion and C$\alpha$-only geometry for Proteina. ESMFold~\citep{lin2023evolutionary} refolds every sequence with two recycles. The cheap verifier computes the proxy score $S$ of Section~\ref{sec:experiments} for each of the 64 sequence--refold pairs, and we select the highest-scoring pair, which determines the winning backbone and sequence. The midpoint state of the winning backbone seeds the exploit candidates of the next round, and the winning sequence is kept for the offline AlphaFold~3 evaluation. Algorithm~\ref{alg:protein-feedback} summarizes this feedback procedure.

\begin{algorithm}[H]
  \caption{Protein feedback and offline evaluation}
  \label{alg:protein-feedback}
  \begin{algorithmic}[1]
    \REQUIRE Motif task $m$, generator $G$, rounds $T=30$, backbones per round $B=8$, sequences per backbone $Q=8$
    \ENSURE Winners $\{(x_t^\star,q_t^\star,S_t^\star)\}_{t=1}^{T}$ and oracle labels $\{Y_t\}_{t=1}^{T}$
    \FOR{$t=1,\ldots,T$}
      \IF{$t=1$}
        \STATE Sample $B$ fresh backbone--midpoint pairs $(x_j,a_j)$ from $G(m)$
      \ELSE
        \STATE Sample $B/2$ backbones $x_j$ by independent continuations from $a$, setting $a_j=a$
        \STATE Sample the remaining $B/2$ backbone--midpoint pairs afresh from $G(m)$
      \ENDIF
      \FOR{$j=1,\ldots,B$}
        \STATE Sample $Q$ sequences $q_{j,1},\ldots,q_{j,Q}$ from ProteinMPNN conditioned on $x_j$
        \FOR{$k=1,\ldots,Q$}
          \STATE Refold $q_{j,k}$ with ESMFold and compute its verifier score $S_{j,k}$
        \ENDFOR
      \ENDFOR
      \STATE Select $(j^\star,k^\star)\in\arg\max_{j,k}S_{j,k}$
      \STATE $(x_t^\star,q_t^\star,S_t^\star)\leftarrow(x_{j^\star},q_{j^\star,k^\star},S_{j^\star,k^\star})$
      \STATE $a\leftarrow a_{j^\star}$
    \ENDFOR
    \FOR{$t=1,\ldots,T$}
      \STATE Run AlphaFold~3 on $q_t^\star$ and compute $Y_t$ against motif $m$ and backbone $x_t^\star$
    \ENDFOR
  \end{algorithmic}
\end{algorithm}

\paragraph{Model hyperparameters.}
\begin{itemize}[leftmargin=*,itemsep=0.4em]
  \item \textbf{RFdiffusion configuration.} We use RFdiffusion~\citep{watson2023novo}.
    \begin{itemize}[leftmargin=1.5em,nosep]
      \raggedright
      \item \textbf{Code.} We use the \href{https://github.com/RosettaCommons/RFdiffusion}{official RFdiffusion GitHub repository}.
      \item \textbf{Checkpoint.} \path{Base_ckpt.pt}.
      \item \textbf{Reverse process.} \path{diffuser.T}=50 and \path{inference.final_step}=1. We retain the step-25 state for subsequent exploit samples.
      \item \textbf{Translations.} \path{diffuser.schedule_type}=\texttt{linear}, \path{diffuser.b_0}=0.01, \path{diffuser.b_T}=0.07, and \path{diffuser.crd_scale}=0.25.
      \item \textbf{Rotations.} \path{diffuser.so3_type}=\texttt{igso3}, \path{diffuser.so3_schedule_type}=\texttt{linear}, \path{diffuser.min_b}=1.5, \path{diffuser.max_b}=2.5, \path{diffuser.min_sigma}=0.02, and \path{diffuser.max_sigma}=1.5.
      \item \textbf{Conditioning.} \path{inference.model_runner}=\texttt{SelfConditioning}, \path{inference.align_motif}=\texttt{True}, \path{inference.recenter}=\texttt{True}, and \path{preprocess.motif_sidechain_input}=\texttt{True}.
    \end{itemize}
  \item \textbf{Proteina configuration.} We use Proteina~\citep{geffner2025proteina}.
    \begin{itemize}[leftmargin=1.5em,nosep]
      \raggedright
      \item \textbf{Code.} We use the \href{https://github.com/NVIDIA-Digital-Bio/proteina}{official Proteina GitHub repository}.
      \item \textbf{Checkpoint.} \path{ckpt_name}=\path{proteina_v1.7_DFS_60M_notri_motif_scaffolding.ckpt}.
      \item \textbf{Integration.} \path{dt}=0.0025,  \path{schedule.schedule_mode}=\texttt{log} and \path{schedule.schedule_p}=2.0.
      \item \textbf{Stochastic term.} \path{sampling_caflow.sampling_mode}=\texttt{sc}, \path{sampling_caflow.sc_scale_noise}=0.4, \path{sampling_caflow.sc_scale_score}=1.0, \path{sampling_caflow.gt_mode}=\texttt{1/t}, \path{sampling_caflow.gt_p}=1.0, and \path{sampling_caflow.gt_clamp_val}=\texttt{null}.
      \item \textbf{Conditioning.} \path{self_cond}=\texttt{False}, \path{fold_cond}=\texttt{False}, \path{guidance_weight}=1.0, and \path{autoguidance_ckpt_path}=\texttt{null}. 
      \item \textbf{Feedback state.} We retain the C$\alpha$ flow state nearest $t=0.5$ for subsequent exploit samples.
    \end{itemize}
\end{itemize}

\paragraph{RMSD computations.}
Each RMSD is computed after a separate least-squares rigid alignment using Biopython's \texttt{Superimposer}~\citep{cock2009biopython}. Backbone RMSD compares all C$\alpha$ atoms of the generated backbone and the ESMFold refold. Motif RMSD compares the prescribed motif's N, C$\alpha$, and C atoms with the corresponding atoms in the refold. Scaffold RMSD excludes motif residues for Proteina but includes the entire generated chain for RFdiffusion, where it equals backbone RMSD.

For the ground-truth label, we run AlphaFold~3~\citep{abramson2024accurate} on the winning sequence of each round after all trajectories have been generated. Using the recorded motif positions and scaffold mask, a proposal is successful if its top-ranked prediction has motif N--C$\alpha$--C RMSD at most 0.75~\AA{}, scaffold-mask C$\alpha$ RMSD at most 1.5~\AA{}, pTM at least 0.80, and mean pLDDT at least 85.

\paragraph{Roundwise likelihood-ratio estimation.}
At round $t$, we use the feature $X_{i,t}=(S_{i,t-1},S_{i,t})\in\mathbb{R}^2$, with $S_{i,0}=0$. We fit a separate estimator for each generator, deployment fold, split seed, and round. Let $\mathcal{A}_t$ and $\mathcal{N}_t$ index evidence trajectories with AlphaFold~3 labels one and zero, respectively. We standardize each feature coordinate using only the observations indexed by $\mathcal{A}_t$.

On the standardized features, we use the Gaussian kernel
\begin{equation*}
  K_h(x,x')
  =
  \exp\left(-\frac{\lVert x-x'\rVert_2^2}{2h^2}\right),
  \qquad h=1.
\end{equation*}
For every coordinate $a$ in the augmented null pool $\widetilde{\mathcal N}_t$, the alternative and leave-one-out null density estimates are
\begin{equation*}
  \widehat p_{1,t}(X_{a,t})
  =
  \frac{1}{|\mathcal A_t|}
  \sum_{i\in\mathcal A_t}K_h(X_{a,t},X_{i,t}),
  \qquad
  \widehat p_{0,t,-a}(X_{a,t})
  =
  \frac{1}{|\widetilde{\mathcal N}_t|-1}
  \sum_{j\in\widetilde{\mathcal N}_t\setminus\{a\}}
  K_h(X_{a,t},X_{j,t}).
\end{equation*}
We define
\begin{equation*}
  \widehat L_{t,-a}(X_{a,t})
  =
  \operatorname{clip}_{[10^{-3},10^3]}
  \left(
    \frac{\widehat p_{1,t}(X_{a,t})+10^{-6}}
         {\widehat p_{0,t,-a}(X_{a,t})+10^{-6}}
  \right)
\end{equation*}
and apply the index-betting normalization in Equation~\ref{eq:split-evalue}. If either evidence class is empty, we set $E_t=1$. This case did not occur in the reported evaluation.

We use the 100 calibration trajectories to select the online $e$-BH scale through robust-core CRC. The held-out labels are used only to evaluate stopped FDP.

We evaluate $\alpha\in\{0.05,0.10,0.15,0.20\}$ and $\kappa\in\{1,5,10\}$ using the scale grid
\begin{equation*}
  \Lambda
  =
  \bigl\{1.5^j \mathrel{\big|} j\in\{0,\ldots,7\}\bigr\}
  \cup
  \{20\}.
\end{equation*}
 
\end{document}